\documentclass[11pt]{article}

\usepackage{arxiv}
\usepackage{booktabs}
\usepackage{graphicx}
\usepackage{amsmath}
\usepackage{amssymb}
\usepackage{amsthm}
\usepackage[hidelinks]{hyperref}
\usepackage{url}
\hypersetup{
  pdftitle={Task- and Session-Level Model Routing: A Common-Interface Hybrid Evaluation of Four Open-Source Routers Across Four Benchmarks},
  pdfauthor={Kiran Kumar and Santhosh Kumar Saminathan},
  pdfsubject={Evaluation of open-source LLM routers},
  pdfkeywords={LLM routing, agentic systems, benchmark evaluation, cost-quality tradeoff}
}

\newtheorem{definition}{Definition}
\newtheorem{lemma}{Lemma}
\newtheorem{proposition}{Proposition}

\title{Task- and Session-Level Model Routing: A Common-Interface Hybrid Evaluation of Four Open-Source Routers Across Four Benchmarks}
\author{
  Kiran Kumar \\
  Indiana University, Bloomington, IN, USA \\
  \texttt{knkumar@iu.edu}
  \and
  Santhosh Kumar Saminathan \\
  Independent Researcher \\
  \texttt{s13.santhosh@gmail.com}
}
\date{}

\begin{document}
\maketitle

\begin{abstract}
Agentic systems increasingly delegate model selection to a router, yet open-source routers are usually evaluated with different tasks, candidate pools, and execution protocols, limiting direct comparison. We present a common measurement protocol and hybrid evaluation of four router implementations across RouterBench, BFCL v4, tau2-bench, and WebArena. We evaluate 290 frozen tasks against a locked matrix of 2,610 candidate outcomes. Three routers emit constant or near-constant tier assignments; only vLLM Semantic Router varies materially with prompt content, and it has the highest observed success rate on none of the four benchmarks. Always-Mid matches Aurelio exactly on three benchmarks and within 0.003 on the fourth. For vLLM, task-level superiority tests detect no task-specific advantage over a share-matched content-blind allocation; equivalence is established only on WebArena at the protocol-declared five-percentage-point margin. The results show that, under these configurations and controls, observed gains track selected-tier composition more closely than demonstrated task-specific targeting. Fixed-tier baselines and selected-tier distributions are therefore necessary controls in router evaluation; the findings are scoped to these configurations, candidate pool, and frozen benchmark samples, not to routing paradigms in general.
\end{abstract}


\providecommand{\appendices}{\appendix}
\providecommand{\appendices}{\appendix}
\section{Introduction}

Model routing is a consequential design choice in modern LLM agents. Each request, and in systems that support it each step of a multi-step trajectory, can be dispatched to one of several candidate models, tools, or policies \cite{routellm}, trading off task quality, dollar cost, latency, and, in agentic settings, tool-call correctness and multi-step policy adherence. The evaluated routers operate at a coarser granularity: they select one model per task or session (Table \ref{tab:granularity}). We therefore do not evaluate per-step routing within a trajectory. When task difficulty varies, routing can reduce cost by assigning suitable tasks to cheaper models while reserving stronger models for tasks on which model choice materially affects success.

An ecosystem of open-source routers has emerged to automate this choice: research systems trained on human preference data \cite{routellm}, production API gateways with fallback and budget controls \cite{litellm}, and embedding-based semantic dispatchers \cite{semanticrouter,vllmsemantic}, among others \cite{llmrouter,nvidiablueprint}. Each targets a different point in the space of cost, quality, latency, and reliability, and these systems are commonly evaluated with different tasks, candidate pools, metrics, and execution protocols, which limits direct comparison.

These heterogeneous evaluations make it difficult to select a router at a specified operating point (quality, cost, latency, tool reliability, task domain) or to judge how sensitive that selection is to the task domain. A single-benchmark evaluation cannot estimate this cross-domain sensitivity. Answering both parts requires four things at once: (a) a common interface any router can sit behind, (b) a benchmark suite spanning single-turn QA through long-horizon agentic tasks, (c) a metrics protocol that reports Pareto frontiers and cross-benchmark consistency instead of a single leaderboard number, and (d) faithful execution of each router and benchmark mode, because behavior under a modeled approximation may not transfer to production execution.

Run side by side this way, at package defaults wired through our specified adapters and on a fixed candidate pool, the routers' more elaborate selection mechanisms produce little separation in the graders' verdicts. Several configuration choices these packages leave open are ours, not the vendors': Aurelio's three reference utterances per tier, its easy/medium/hard framing, and its mid-general fallback default are adapter choices we supplied, and RouteLLM's package-default escalation threshold is applied to a weak/strong candidate pair it was not calibrated on (Appendix, Router Configurations). Aurelio's high fallback rate below therefore follows partly from that sparse utterance set, and should not be read as a property of the packaged router. The claim is scoped to these configurations on this pool, not to the routing paradigms in the abstract.

This paper provides six outputs, matching the deliverables of the accompanying software artifact \texttt{router\_\allowbreak{}benchmark/}:
\begin{enumerate}
\item A measurement protocol for router evaluation, addressing both rank aggregation and task attribution. Proposition \ref{prop:rank} settles the \emph{aggregation} half: once two routers' benchmark rankings invert, no benchmark-independent scalar preserves every benchmark-specific ordering, so per-benchmark reporting is necessary whenever those individual orderings are what a reader needs, and any scalar ranking is well-defined only relative to a traffic weighting or deployment utility declared in advance. The \emph{attribution} half is a separate methodological requirement that the proposition neither implies nor supports: a claim that a router assigns the right task to the right tier must be tested against a share-matched content-blind allocation, because a router's aggregate tier mixture and its task-by-task targeting are distinct sources of any advantage it shows (Sections \ref{sec:protocol} and \ref{sec:results}).
\item A \texttt{Router}/\texttt{Benchmark} adapter interface and an \texttt{EvaluationHarness} that evaluates any list of routers against any list of benchmarks, pair by pair (Section \ref{sec:interface}).
\item A deployment-focused metric suite, each statistic mapped to a decision it informs, plus a Pareto frontier computation over routers (Section \ref{sec:interface}).
\item A hybrid evaluation of four open-source routers against four benchmarks from a curated research shortlist: RouterBench replays logged outcomes, while BFCL v4, tau2-bench, and WebArena execute live model calls through one harness, metrics protocol, and plotting pipeline (Sections \ref{sec:method} and \ref{sec:results}).
\item A principal empirical result. Under the evaluated candidate pool, adapters, defaults, and frozen task samples, a single content-free fixed-tier policy, Always-Mid, reproduces the aggregate-leading router Aurelio Semantic Router's observed success on every benchmark, matching it exactly on RouterBench, BFCL, and WebArena and trailing it by 0.003 on tau2-bench. The wider set of fixed-tier heuristics (Always-Strongest, Always-Cheapest, and the constant-tier routers themselves) matches or exceeds the highest observed router success on three of the four benchmarks (Section \ref{sec:results}). Paired with a share-matched permutation test, this yields the narrower claim the data supports: we detect no statistically significant task-specific success advantage on any benchmark, and equivalence to a content-blind allocation with the same tier shares is \emph{established} only on WebArena, only at the protocol-declared $\pm0.05$ margin, and not at $\pm0.03$; RouterBench, BFCL, and tau2-bench remain inconclusive. The formal share-matched test is run for vLLM Semantic Router, the only router whose assignment varies materially; the near-constant routers are covered instead by a direct paired equivalence test against the fixed tier they reproduce, which is exact wherever their selection is constant.

That is not the same as saying no router adds anything a fixed tier cannot. vLLM Semantic Router's WebArena mixture produces an intermediate cost-quality operating point that no single fixed tier reproduces, and it is nondominated there; what our tests do not detect is a benefit from \emph{which} task gets which tier, as distinct from the mixture itself. Because three of the four routers are constant or near-constant on this pool, the raw-success comparison is largely one fixed tier against another, and it bears on these configurations rather than on task-adaptive routing as a paradigm.
\item A cross-benchmark consistency analysis. Relative rank inverts between benchmarks, an instance of Proposition \ref{prop:rank}, resolved on RouterBench and directional but not statistically settled on BFCL. Rank variance is uninformative in this suite: every router is rank-inverted on exactly one axis, so all four carry an identical variance (0.00 pooled, 0.75 strict). That leaves mean rank as the only separating axis, and the strength of the lead it shows is grouping-dependent even though the identity of the leader is not (Section \ref{sec:results}).
\end{enumerate}

Section \ref{sec:related} reviews related work. Section \ref{sec:protocol} formalizes the evaluation problem. Section \ref{sec:interface} describes the interface and metrics. Section \ref{sec:method} describes the hybrid execution methodology. Section \ref{sec:results} reports results, per benchmark and across benchmarks. Section \ref{sec:discussion} discusses implications and threats. Section \ref{sec:conclusion} concludes.

\section{Related Work}
\label{sec:related}

Prior routing work spans several partly overlapping lines, each addressing a different part of the problem and typically evaluated on a different benchmark. We group it accordingly.

\textbf{Query-level model routing.} RouteLLM \cite{routellm} casts routing as binary weak/strong model selection, trains on human preference data, and evaluates on MMLU and GSM8K; it reports the cost-quality tradeoff curve that this paper adopts as a primary axis. We include RouteLLM as the query-level representative and the only trained preference model among the four evaluated routers. Other query-level systems we cite but do not evaluate include LLMRouter \cite{llmrouter} and the Anyscale tutorial line, which pursue the same query-level selection in smaller models. FrugalGPT \cite{frugalgpt} studies budget-constrained LLM cascades, which we use as the cost-aware baseline concept other than a router.

\textbf{Production routing and gateways.} LiteLLM's router \cite{litellm} treats routing as infrastructure: OpenAI-compatible proxying, load balancing, budget enforcement, and fallback chains, plus usage-based and adaptive policies. LiteLLM Router represents the open-source production-gateway paradigm in our evaluation. We cite the NVIDIA AI Blueprint LLM Router \cite{nvidiablueprint} to show the paradigm is not confined to open-source tooling. Both emphasize gateway functions such as reliability, observability, load balancing, and fallback rather than learned task-difficulty estimation; our per-benchmark results (Section \ref{sec:results}) measure the consequences of that design scope for the routing-behavior analysis this paper reports.

\textbf{Semantic and intent routing.} The vLLM Semantic Router \cite{vllmsemantic} and Aurelio Semantic Router \cite{semanticrouter} route by embedding similarity, matching a prompt to a signal (privacy, cost, latency, safety) or intent \emph{before} generation. This contrasts with LLM-Blender's \cite{llmblender} post-hoc ensembling, which ranks candidate \emph{outputs after} generation; our shortlist deliberately keeps routing before execution separate from output selection after it.

\textbf{Self-evolving and theoretical routing.} EvoRoute \cite{evoroute} routes \emph{inside} long-horizon agent trajectories, refining its own policy over repeated task exposure and reporting Pareto-optimal cost/accuracy/latency tradeoffs. Where EvoRoute is an algorithm operating within an agent's trajectories, ours is an infrastructure and evaluation contribution: a common interface over existing third-party routers, evaluated across one replay benchmark and three live execution benchmarks. Mahmood \cite{mahmoodrouting} models the provider's routing policy and the user's reprompting behavior as a Stackelberg game. We do not test that prediction. We include two deliberately limited cascade controls, but neither is a response-aware cascade that escalates after inspecting a first response; their results are reported in Section \ref{sec:results}.

\textbf{Routing benchmarks.} RouterBench \cite{routerbench}, the closest analogue to a router-only benchmark, logs over 405{,}000 model outcomes and enables offline quality at fixed cost curves without live inference. RouterBench, however, is single-turn and tests neither tool use nor long-horizon agent behavior. RouterEval \cite{routereval} and LLMRouterBench \cite{llmrouterbench} likewise provide large-scale offline records for routing research; LLMRouterBench reports that several routers fail to reliably beat simple baselines under unified evaluation. Shnitzer et al. \cite{shnitzer} repurpose existing datasets to train a router as binary classifiers over single-turn text; RouterBench follows the same strategy of repurposing existing benchmarks. We instead evaluate routers over real runtime payloads: executable tool calls, multi-turn sessions, and browser trajectories, not text collected beforehand.

\textbf{Agentic and tool-use benchmarks.} BFCL \cite{bfcl, bfclv4}, the standard function-calling correctness benchmark, is the natural test of whether a router preserves tool-call correctness when it downgrades to a cheaper model. The tau-bench/tau2-bench line \cite{taubench, tau2bench} tests multi-turn, policy-constrained customer-service agents, probing task \emph{consistency} across turns. WebArena \cite{webarena} evaluates web-navigation agents on self-hosted sites, where a single bad routing decision can cascade across a long browser trajectory. SWE-bench Verified \cite{swebench} and Terminal-Bench 2.0 \cite{terminalbench} probe a still harder long-horizon regime that falls outside our benchmark scope. Evaluation context can also alter apparent agent performance: Mahmood et al. \cite{agentjudge} localize an Agent-as-a-Judge prompt stack across five languages and find that the leading judge backbone changes by language. Their controlled variation differs from our comparison of routers across benchmarks, but it likewise makes the evaluation setting an explicit part of the conclusion.

Unified side-by-side router comparison is not itself new: LLMRouterBench \cite{llmrouterbench} evaluates several routers under one framework and already reports routers failing to reliably beat simple baselines. Unlike prior unified comparisons that score routing decisions against precomputed single-turn records, we evaluate these four implementations (RouteLLM, LiteLLM Router, vLLM Semantic Router, and Aurelio Semantic Router) under a common interface across one replay benchmark and three live execution benchmarks (tool calls, multi-turn policy sessions, and browser trajectories) using a deployment-oriented metrics protocol.

\section{Formal Evaluation Framework}
\label{sec:protocol}

We first formalize the router evaluation problem. We define Pareto dominance and cross-benchmark rank consistency, then show why no benchmark-independent scalar can preserve every benchmark-specific ordering once two routers' rankings invert (Proposition \ref{prop:rank}).

\begin{definition}[Task, candidate, router]
\label{def:task}
A task $t$ is a tuple $(x_t, \tau_t, C_t)$ where $x_t$ is the task input, $\tau_t \in \{0,1\}$ a tool-call requirement flag, and $C_t = \{c_1, \dots, c_k\}$ a finite candidate pool, each candidate characterized by a task-dependent quality and a per-invocation dollar cost and latency. Difficulty descriptors are introduced later as an analysis-only candidate-agreement proxy, not as an input to the routing formalism. A router is a function $R : (t, \mathrm{ctx}, C_t) \mapsto (c^\star, \mathrm{conf}, \mathrm{fb}, \mathrm{meta})$, mapping a task, its context, and the candidate pool to a selected candidate $c^\star \in C_t$, a confidence, a fallback flag, and auxiliary metadata (Section \ref{sec:interface}).
\end{definition}

\begin{definition}[Operating point]
For router $R$ and benchmark $B$, the operating point $m(R,B)$ is the metrics vector defined in Section \ref{sec:interface}, computed by scoring $R$'s decisions on $B$'s task set via $B$'s own $\mathrm{score}(t,\mathrm{decision})$ function; the two coordinates central to this paper are mean success rate $s(R,B) \in [0,1]$ and mean cost per task $c(R,B) \in \mathbb{R}_{\ge 0}$. Zero cost is admissible: a replayed RouterBench cell can carry a zero logged cost, and a transport failure is detected as a zero billed cost.
\end{definition}

\begin{definition}[Pareto dominance and frontier]
\label{def:pareto}
Fix a benchmark or benchmark aggregate $B$. Router $R_1$ dominates router $R_2$, written $R_1 \succ_B R_2$, iff $c(R_1,B) \le c(R_2,B)$ and $s(R_1,B) \ge s(R_2,B)$, with at least one inequality strict. The Pareto frontier $P(\mathcal{R},B)$ over a router set $\mathcal{R}$ is $\{R \in \mathcal{R} : \nexists R' \in \mathcal{R},\, R' \succ_B R\}$.
\end{definition}

\begin{lemma}[Extremal frontier membership]
\label{lem:frontier}
Let $\mathcal{R}$ be evaluated on benchmark $B$. If $R_c = \arg\min_{R\in\mathcal{R}} c(R,B)$ is the unique cost minimizer, then $R_c \in P(\mathcal{R},B)$; symmetrically, if $R_s = \arg\max_{R\in\mathcal{R}} s(R,B)$ is the unique success maximizer, then $R_s \in P(\mathcal{R},B)$.
\end{lemma}
\begin{proof}
Suppose, for contradiction, some $R'$ dominates $R_c$, i.e. $R' \succ_B R_c$. Definition \ref{def:pareto} requires $c(R',B) \le c(R_c,B)$. Since $R_c$ is the unique cost minimizer, no other router can match or beat its cost, so this inequality forces $R' = R_c$. But $R' \succ_B R_c$ makes $R'$ and $R_c$ the same router, and Definition \ref{def:pareto} also requires at least one strict inequality between them, which is impossible when they are identical. So no dominator $R'$ exists and $R_c \in P(\mathcal{R},B)$. The argument for the unique success maximizer $R_s$ is identical, with cost and success swapped.
\end{proof}

Lemma \ref{lem:frontier} is why several frontier memberships in Section \ref{sec:results} are \emph{guaranteed}, not merely observed: LiteLLM Router (the unique cost minimizer \emph{among the four routers} on every benchmark) and the unique success maximizer on each benchmark must lie on their respective frontiers. The uniqueness is relative to the router set: once fixed-tier policies enter the universe, the cheap-only Always-Cheapest matches LiteLLM's cost, so neither is the strict minimizer and both share the guaranteed seat (Section \ref{sec:results}).

\begin{proposition}[Benchmark dependence of router rank]
\label{prop:rank}
Let $R_a, R_b$ be routers and $B_1, B_2$ benchmarks with $s(R_a,B_1) > s(R_b,B_1)$ but $s(R_a,B_2) < s(R_b,B_2)$. Then no benchmark-independent total order $\pi$ on $\{R_a,R_b\}$ agrees with the success-rate order induced by both benchmarks. Consequently, no single \emph{benchmark-independent} scalar (whether success rate on one benchmark or a fixed-weight average across benchmarks) can preserve both benchmark-specific rankings simultaneously. A scalar ranking is still well defined once a weighting or deployment utility is declared, but it is then weighting-dependent rather than a property of the routers alone. Per-benchmark reporting is therefore necessary whenever the individual benchmark orderings are what a reader needs to retain, which is not an expository convenience but the only way to keep them.
\end{proposition}
\begin{proof}
The first benchmark requires $\pi$ to rank $R_a$ above $R_b$, while the second requires the reverse. No total order can satisfy both. Any scalar induces such an order, so no scalar can preserve both rankings simultaneously.
\end{proof}

Proposition \ref{prop:rank} formalizes an aggregation constraint: once two systems' benchmark-specific orders invert, no benchmark-independent scalar can preserve both orders simultaneously. Benchmark-independent scalars still exist, and we report several below, but each induces an order that disagrees with at least one benchmark-specific order unless its traffic weighting or deployment utility is declared. This aggregation result does not imply the separate attribution requirement to compare task-level assignments with share-matched content-blind baselines; that requirement is evaluated in Section \ref{sec:results}.

Section \ref{sec:results} instantiates Proposition \ref{prop:rank} with observed, not modeled, data. Take $R_a = $ Aurelio Semantic Router, $R_b = $ LiteLLM Router, $B_1 = $ RouterBench, $B_2 = $ BFCL: the observed success rates satisfy $s(\text{Aurelio},B_1) = 0.683 > s(\text{LiteLLM},B_1) = 0.300$ yet $s(\text{Aurelio},B_2) = 0.811 < s(\text{LiteLLM},B_2) = 0.833$. The hypothesis is met by the observed point estimates, so the observed reversal is a property of those point estimates rather than an artifact of how they are aggregated. Its two legs are resolved unequally, however. The RouterBench gap ($0.683$ vs.\ $0.300$) is large and well separated, whereas the BFCL gap ($0.811$ vs.\ $0.833$; Aurelio minus LiteLLM $-0.022$, [$-0.133$, $0.100$], $p=1.000$) is not statistically distinguishable under the primary adjusted test. We therefore read this as a resolved reversal on RouterBench and a directional, statistically unresolved one on BFCL, not an established population-level rank flip on both axes.

\begin{definition}[Cross-benchmark consistency]
\label{def:consistency}
Given router set $\mathcal{R}$ and benchmark set $\mathcal{B} = \{B_1, \dots, B_n\}$, let $\mathrm{rank}(R,B_i) \in \{1,\dots,|\mathcal{R}|\}$ denote $R$'s rank by success rate on $B_i$ (1 = best; tied routers receive the average of the tied rank positions). Define $R$'s mean rank $\bar\rho(R) = \frac{1}{n}\sum_i \mathrm{rank}(R,B_i)$ and its rank variance $\nu(R) = \frac{1}{n}\sum_i (\mathrm{rank}(R,B_i) - \bar\rho(R))^2$. A router with low $\bar\rho(R)$ and low $\nu(R)$ is both strong and consistent across the benchmark set; low $\bar\rho(R)$ alone, the quantity a single benchmark leaderboard reports, cannot distinguish a consistently good router from one that is excellent on the benchmarks it happens to be evaluated on and poor elsewhere.
\end{definition}

Definition \ref{def:consistency} supplies the metric Section \ref{sec:results} uses to organize this paper's central empirical claim.

\section{Implementation and Reported Metrics}
\label{sec:interface}

\subsection{Common Adapter Interface}
The abstractions above map onto executable adapters as follows. Each router implements a single-method contract corresponding to Definition \ref{def:task}: a \texttt{route} call over a task, its context, and the candidate pool returns a route decision consisting of a selected candidate, a confidence, a fallback flag, and metadata. Each benchmark implements two methods: \texttt{generate\_tasks()}, which emits a fixed list of \texttt{Task} objects (each carrying a domain, a difficulty in $[0,1]$, a tool-call requirement flag, and a candidate pool), and \texttt{score(task, decision)}, which returns success, cost, latency, and, where applicable, tool-call correctness. This mirrors the shortlist's own recommended protocol: normalize all routers behind one adapter, log the candidate set, route decision, confidence, fallback, cost, tokens, and final outcome.

\subsection{Evaluation Harness}
The harness method \texttt{evaluate} takes any list of router adapters and any list of benchmark adapters, evaluates every (router, benchmark) pair over a configurable number of trials with deterministic per (router, benchmark, trial) seeding, and returns a tidy table, one row per (router, benchmark, task, trial). This is the single point of integration: adding a router or benchmark requires no changes to the harness, metric, or plotting code.

\subsection{Metrics}
For every (router, benchmark) pair we compute a deployment-focused suite, with each statistic linked to a deployment decision: success rate, the primary task outcome axis $s(R,B)$; cost per task and cost per successful task (USD), reported separately because they answer different questions; latency; a fallback rate; a route-stability measure across repeated trials; and, on the tool-required benchmarks, tool-call accuracy. Aggregating across benchmarks per router, we add the Pareto frontier of Definition \ref{def:pareto} and the cross-benchmark mean rank and rank variance of Definition \ref{def:consistency}. The last two add what a single benchmark leaderboard cannot report: how much a router's standing moves as the task domain changes.

\section{Methodology}
\label{sec:method}

\subsection{Router and Benchmark Selection}
We evaluate four routers (RouteLLM \cite{routellm}, LiteLLM Router \cite{litellm}, vLLM Semantic Router \cite{vllmsemantic}, and Aurelio Semantic Router \cite{semanticrouter}) against four benchmarks: RouterBench \cite{routerbench}, BFCL v4 \cite{bfcl, bfclv4}, tau2-bench \cite{taubench, tau2bench}, and WebArena \cite{webarena}. The four routers cover three common routing paradigms. \emph{Static rules and heuristics}: in the evaluated configuration, LiteLLM Router uses programmatic logic, load balancing, and cost-limit fallbacks, and does not inspect prompt content. \emph{Semantic similarity}: vLLM and Aurelio Semantic Routers embed the incoming prompt and match it against a database of intents or historical queries. \emph{Trained preference models}: RouteLLM uses ML classifiers (matrix factorization) to predict whether a cheaper model can match a stronger one on the specific prompt. These paradigms differ in the information available to a routing decision, static rules, topical matching, or capability prediction, a distinction the results (Section \ref{sec:results}) return to.

\subsection{Task Selection}
\label{sec:task-selection}
Benchmark task selection can materially affect router comparisons, so Table \ref{tab:task-selection} states it per benchmark: the eligible population, how many tasks were drawn from it, by what method, under what seed, and what the population excludes. Two of the four are uniform random samples under a fixed seed and two are deterministic prefixes of a filtered pool.

None is stratified: we did not balance the sample on difficulty, domain, or any outcome-correlated attribute, so RouterBench and BFCL are unstratified random samples from their eligible populations. The tau2-bench and WebArena sets are deterministic convenience prefixes taken in shipped file order, chosen so that the selected task IDs are reproducible without a seeded draw over a pool whose order may itself encode domain or difficulty. They therefore do not support design-based generalization to the full tau2-bench or WebArena populations, and every inference below is conditional on the evaluated task IDs.

The eligible populations are narrower than the benchmarks themselves, and the narrowing is a scope decision rather than a result. BFCL v4 contributes only its single-turn \texttt{simple} categories, which is the slice our grader covers, so the multi-turn, parallel, and irrelevance categories never entered the pool. tau2-bench contributes only its retail domain. WebArena contributes only tasks that run on exactly one of the two sites we host under pinned images, gitlab and shopping, which reduces its 812 tasks to 367 before any selection. RouterBench alone is sampled from its full published table.

Sample sizes were set in advance based on API cost constraints rather than a power calculation, yielding 30 (BFCL), 60 (RouterBench), and 100 (tau2-bench and WebArena) tasks. These sizes limit the statistical resolution of the pairwise comparisons in Section \ref{sec:results}.

No tasks were dropped after selection. The candidate matrix contains all $290 \times 3 \times 3 = 2{,}610$ candidate-outcome rows (290 tasks $\times$ 3 candidates $\times$ 3 replicates) with no missing cells, cache-served results, or execution failures. Two related events are recorded rather than excluded. A provider that rejects a BFCL function schema returns no tool call, and the harness grades that as a failed task instead of dropping it. WebArena's candidate and route rows were regenerated after a browser-revision repair, which replaced those rows rather than removing tasks; the frozen task IDs are unchanged and the other three benchmarks' rows were retained.

\begin{table*}[t]
\caption{Task selection per benchmark. \emph{Eligible population} is the set a task could be drawn from after the scope filters in the last column; the benchmarks' full sizes are 36{,}497, 550 in the covered categories, 114, and 812. No sample is stratified, and no task was removed after selection.}
\label{tab:task-selection}
\centering
\small
\setlength{\tabcolsep}{4pt}
\begin{tabular}{@{}lp{0.21\textwidth}rp{0.16\textwidth}lp{0.25\textwidth}@{}}
\toprule
Benchmark & Eligible population & Sel. & Selection method & Seed & Population exclusions \\
\midrule
RouterBench & 36{,}497 logged prompts & 60 & Uniform random, without replacement & 1234 & None; the full published table is eligible \\
BFCL v4 & 550 single-turn \texttt{simple} items (400 Python, 100 Java, 50 JavaScript) with published ground truth & 30 & Uniform random, without replacement & 1234 & Multi-turn, parallel, and irrelevance categories. The adapter also drops any item lacking shipped ground truth, which removes none of the 550 \\
tau2-bench & 114 retail-domain tasks & 100 & First 100 in shipped file order & Deterministic & All non-retail domains \\
WebArena & 367 single-site gitlab or shopping tasks & 100 & First 100 in pool order (66 shopping, 34 gitlab) & Deterministic & Multi-site tasks; all sites without a pinned self-hosted image \\
\bottomrule
\end{tabular}
\end{table*}

\subsection{Hybrid Execution Protocol}
The protocol is hybrid in a specific sense: for every benchmark we build a shared outcome matrix first and then replay each router's selected tier against that matrix, so router decisions are live but the graded outcome they inherit is a saved one. What differs across benchmarks is how the matrix was produced. RouterBench uses outcomes logged by the benchmark's authors, so no new model call is made for the answer. For BFCL v4, tau2-bench, and WebArena, we execute every candidate model on every frozen task to populate the shared three-tier outcome matrix. Router decisions are then replayed by joining the selected tier to its saved candidate outcomes. The candidate pool is cheap-small $=$ gpt-5.4-nano, mid-general $=$ claude-sonnet-4-6, and strong-frontier $=$ claude-opus-4-8, under a pricing snapshot dated 2026-07-02. The router packages run through the common interface of Section \ref{sec:interface}; the adapters implement \texttt{route()}, \texttt{generate\_tasks()}, and \texttt{score()} against the specified task data and graders. Every task-candidate cell stores three outcome replicates and each router takes two routing trials per task; replay joins each selected route to a saved candidate outcome, fixing the outcome seen by each router decision and preventing a router-specific call from changing the selected tier's score. The comparative matrix contains $290 \times 3 \times 3 = 2{,}610$ candidate-outcome rows and $290 \times 4 \times 2 = 2{,}320$ route rows; prompt-output caching is disabled so replicates do not inherit one another's outcomes. The joined-row count $N_{\mathrm{joined}}$ represents tasks $\times$ two routing trials $\times$ three replicates, while the task count $T$ is the primary unit of inference (Section \ref{sec:estimand}).

This design is candidate generation plus counterfactual route replay rather than fully online end-to-end router execution, and the difference bounds what the comparison can show. Because every router draws from one shared matrix, no router can be advantaged or disadvantaged by generation noise specific to its own calls. In exchange, the replay does not capture router-specific prompt transformations, dynamic fallback after generation, request timing, provider load, or context conditioned on the route taken. A single saved outcome is also reused across every router decision that selects the same cell, which is what makes the comparison paired but also means the replicates are shared rather than independent across routers.

Table \ref{tab:granularity} tabulates a distinction the prose leaves implicit: at what granularity each router decides, and what each grader checks. Every benchmark routes at the task or session level, never per turn or per step within a task. On WebArena, \texttt{webarena\_live.py} selects one model before the environment executes the entire multi-step trajectory; on tau2-bench, the router picks one \texttt{--agent-llm} for the whole multi-turn session. We audited all four adapters against their package traces, checking that each per-task decision was produced by the package's own mechanism rather than by adapter-side substitution or approximation; the checks and their outcomes are recorded with the released artifact (Appendix, Reproducibility). The WebArena environment uses a pinned browser build (Playwright 1.32.1, Chromium revision 1055); a revision mismatch in an earlier execution was repaired before this run, and only WebArena candidate and route rows were regenerated while the three other benchmarks' validated rows were retained unchanged.

\begin{table}[t]
\caption{Routing granularity and grading per benchmark.}
\label{tab:granularity}
\centering
\small
\begin{tabular}{@{}lll@{}}
\toprule
Benchmark & Route frequency & Grader \\
\midrule
RouterBench & Once per task & Logged-outcome lookup \\
BFCL v4 & Once per task & Exact/membership check \\
tau2-bench & Once per session & tau2 reward + action check \\
WebArena & Once per task & WebArena official evaluator \\
\bottomrule
\end{tabular}
\end{table}

\subsection{Estimand and Inference}
\label{sec:estimand}
\textbf{Estimand.} Our primary estimand is the mean paired per-task success difference over the frozen benchmark tasks,
\[
\delta \;=\; \frac{1}{T}\sum_{i=1}^{T}\bigl(\mathbb{E}[y_A(i)] - \mathbb{E}[y_B(i)]\bigr),
\]
where $T$ is the number of frozen tasks and $y_A(i)$, $y_B(i)$ are the graded outcomes of policies $A$ and $B$ on task $i$. Each inner expectation is taken over a task-candidate cell's three outcome replicates and over the router's two routing trials, and is estimated by the observed average of those replicates and trials. We treat the replicates as samples from the configured generation and routing procedures, so inference is conditional on the selected tasks, candidate pool, adapters, configurations, and execution period rather than generalizing to the benchmarks' full populations. Every reported point estimate is the plug-in version of $\delta$.

\textbf{Interval construction.} The interval is a hierarchical paired bootstrap over the frozen tasks, matched to that target. One draw resamples tasks with replacement and then independently resamples output replicates within each distinct task-candidate cell. Both policies receive the same task draw; they share replicate draws only when they select the same saved cell, because replicate ordinals from different models or providers are not paired observations. The two routing trials are averaged through their selected cells within each sampled task. We use both stages to represent task-sampling variation and observed within-cell output variation, under an assumption that replicates are exchangeable within each task-candidate cell. Resampling tasks alone would omit the output randomness averaged by the estimate, whereas drawing a single replicate per task would charge that randomness at full weight to a quantity already averaged over three replicates. The procedure does not represent shared temporal, provider, or infrastructure variation across calls. RouterBench is unaffected either way because its replicates are identical by construction. Because the task sets are frozen rather than redrawn, these are conditional resampling intervals over the evaluated task IDs; the effective sample size is the number of frozen tasks (30, 60, 100, 100) rather than the joined-row count $n$ (180, 360, 600, 600).

\textbf{Superiority test.} The primary test is a two-sided task-level sign-flip test on the same per-task differences averaged by the point estimate. Its null requires the task-level difference distribution to be symmetric about zero. That assumption is not guaranteed by the design, because router assignments are deterministic rather than randomized. We therefore interpret it as model-based inference conditional on the frozen task set, not as a design-based randomization test. Ten thousand independent sign flips form the reference distribution, with an add-one correction. A zero difference is invariant under a sign flip, so no tie convention is required. We report three sensitivity tests on the same data and Holm-adjust them within the same family: a bootstrap sign-tail measure, an exact paired sign test on task-level discordant counts, and a null-centred cluster bootstrap. The exact sign test discards effect magnitude and therefore does not require symmetry of magnitudes, although it answers the narrower sign-null question. The four constructions agree on 22 of 24 router-pair comparisons. For the two WebArena vLLM-versus-cheap comparisons, only the sign-tail sensitivity yields an adjusted $p<0.05$; the primary sign-flip test and the other two sensitivities do not, so those effects are treated as unresolved.

\textbf{Equivalence test.} Failing to reject the superiority null is not evidence that an effect is absent, so equivalence is tested directly rather than read off a non-significant $p$. We use two one-sided tests (TOST) on the paired per-task differences, with null $H_0: |\delta| \ge \Delta$ and a primary margin $\Delta = 0.05$; a verdict of practical equivalence requires the adjusted TOST $p$ to fall below $0.05$. The $\pm0.05$ margin is a protocol-declared primary equivalence margin rather than derived from an external standard of a minimum important difference; tighter absolute and benchmark-relative margins are reported as sensitivities in Table \ref{tab:margin-sensitivity}.

\textbf{Multiplicity and reporting conventions.} Intervals are marginal intervals for a single comparison, at 95\% for superiority contrasts and 90\% for the TOST contrasts they accompany. Superiority $p$-values are Holm-adjusted across the six protocol-declared within-benchmark router pairs, and equivalence $p$-values are Holm-adjusted within their own family across the four benchmarks; the two families are adjusted separately, and each margin in the sensitivity table forms its own family. A marginal interval can therefore exclude zero while the multiplicity-adjusted $p$-value does not clear $0.05$, which we flag where it occurs. Throughout, we describe an ordering of point estimates as \emph{observed}, reserve \emph{statistically distinguishable} for a contrast clearing the primary adjusted test, use \emph{equivalent within $\Delta$} only for a cleared TOST, and call a comparison \emph{unresolved} when the evidence supports neither.

\section{Results}
\label{sec:results}

This section proceeds in four steps. We first describe the candidate substrate every router selects over, then report per-benchmark router outcomes, then place those outcomes against fixed-tier controls and cost- and latency-aware summaries, and finally test whether any observed advantage is attributable to task-specific targeting rather than to tier shares. Reading the steps in that order matters, because the controls determine how much of a router's outcome its selection rule can be credited with.

Table \ref{tab:candidate-tiers} reports candidate-tier success before any routing. Three benchmarks share one ordering: mid-general has a substantially higher observed success rate than cheap-small (at least double on RouterBench and WebArena, about $1.5\times$ on tau2-bench). The increment from mid to strong is smaller than the increment from cheap to mid on RouterBench (0.117 versus 0.283), zero on WebArena, and slightly negative on tau2-bench. BFCL v4 shows the opposite ordering, with cheap-small attaining the highest observed success rate. Throughout, ``strong-frontier'' names a price tier rather than a claim about capability on a given benchmark; BFCL's reversal is why the two must be kept apart. This ordering motivates the fixed-tier controls examined below.

\begin{table}[t]
\caption{Candidate success before routing in the four benchmarks. Each row includes three outcomes per task.}
\label{tab:candidate-tiers}
\centering
\small
\setlength{\tabcolsep}{4pt}
\begin{tabular}{@{}llrr@{}}
\toprule
Benchmark & Candidate & n & Success \\
\midrule
RouterBench & cheap / mid / strong & 180 each & .300 / .683 / .800 \\
BFCL v4 & cheap / mid / strong & 90 each & .833 / .811 / .800 \\
tau2-bench & cheap / mid / strong & 300 each & .543 / .810 / .797 \\
WebArena & cheap / mid / strong & 300 each & .110 / .220 / .220 \\
\bottomrule
\end{tabular}
\end{table}

\begin{table*}[t]
\caption{Canonical joined router outcomes across the four benchmarks (success, mean selected-candidate cost, mean router-service cost, all USD).}
\label{tab:router-main}
\centering
\resizebox{\textwidth}{!}{\begin{tabular}{@{}llrrrr@{}}
\toprule
Benchmark & Router & Joined rows & Success & Candidate USD & Router USD \\
\midrule
BFCL v4 & Aurelio Semantic & 180 & 0.811 & 0.0036 & 0.0100 \\
BFCL v4 & LiteLLM & 180 & 0.833 & 0.0001 & 0.0000 \\
BFCL v4 & RouteLLM & 180 & 0.833 & 0.0001 & 0.0100 \\
BFCL v4 & vLLM Semantic & 180 & 0.772 & 0.0019 & 0.0500 \\
RouterBench & Aurelio Semantic & 360 & 0.683 & 0.0003 & 0.0100 \\
RouterBench & LiteLLM & 360 & 0.300 & 0.0000 & 0.0000 \\
RouterBench & RouteLLM & 360 & 0.300 & 0.0000 & 0.0100 \\
RouterBench & vLLM Semantic & 360 & 0.400 & 0.0004 & 0.0500 \\
WebArena & Aurelio Semantic & 600 & 0.220 & 0.0873 & 0.0100 \\
WebArena & LiteLLM & 600 & 0.110 & 0.0021 & 0.0000 \\
WebArena & RouteLLM & 600 & 0.110 & 0.0021 & 0.0100 \\
WebArena & vLLM Semantic & 600 & 0.150 & 0.0256 & 0.0500 \\
tau2-bench & Aurelio Semantic & 600 & 0.813 & 0.2555 & 0.0100 \\
tau2-bench & LiteLLM & 600 & 0.543 & 0.0047 & 0.0000 \\
tau2-bench & RouteLLM & 600 & 0.543 & 0.0047 & 0.0100 \\
tau2-bench & vLLM Semantic & 600 & 0.730 & 0.2262 & 0.0500 \\
\bottomrule
\end{tabular}
}
\end{table*}

\subsection{RouterBench and BFCL v4}
We start with single-turn QA and function calling, the setting the ecosystem's own leaderboards report. On RouterBench (60 tasks, two trials, three replicates, $N_{\mathrm{joined}}=360$ joined rows), Aurelio Semantic Router selects mid-general for every route and reaches 0.683 success; LiteLLM Router and RouteLLM select cheap-small throughout and reach 0.300; vLLM Semantic Router mixes tiers (81.7\% cheap, 1.7\% mid, 16.7\% strong) for 0.400. Aurelio exceeds each cheap-only router by 0.383, with a 95\% paired-bootstrap interval of [0.233, 0.533] and Holm-adjusted $p<0.01$; it exceeds vLLM by 0.283 [0.133, 0.433], also $p<0.01$. Always-Strongest reaches 0.800 on RouterBench, the highest observed success among the evaluated routers and fixed-tier baselines, and Lemma \ref{lem:frontier} guarantees it a frontier seat as the success maximizer.

BFCL v4 reverses the tier ordering. Across 30 tasks ($N_{\mathrm{joined}}=180$ joined rows), cheap-small has the highest observed tier success at 0.833; LiteLLM Router and RouteLLM select it on every route and also reach 0.833, while Aurelio's constant mid-general selection yields 0.811 (Aurelio minus LiteLLM $-0.022$, [$-0.133$, 0.100], $p=1.000$) and vLLM reaches 0.772. LiteLLM minus vLLM is $+0.061$ [$-0.011$, 0.167], not significant ($p=1.000$); once the trials and replicates are averaged within the 30 tasks, no pair of routers differs significantly on BFCL under any of the four constructions of Section \ref{sec:estimand}. The route-equivalence audit rescored 447 same-tier BFCL outcomes with the pinned \texttt{bfcl-eval==2026.3.23} checker and reports zero conflicts, so the BFCL ordering is not an artifact of router-specific grading. This is the concrete instance of Proposition \ref{prop:rank}: the router with the highest observed RouterBench success has the lower observed success on BFCL, so no \emph{benchmark-independent} scalar can preserve both benchmark-specific orders. Weighted scalars remain perfectly well defined, and we report several below; what the proposition rules out is a scalar that is a property of the routers alone.

\subsection{tau2-bench}
The multi-turn picture matches RouterBench's tier ordering. Across 100 tau2-bench sessions ($N_{\mathrm{joined}}=600$ joined rows), cheap-small reaches 0.543 and mid-general 0.810. Aurelio selects mid-general on 97.0\% of routes and reaches 0.813, exceeding LiteLLM and RouteLLM by 0.270 [0.180, 0.360], $p<0.01$; vLLM mixes cheap/mid/strong at 36/44/20\% for 0.730, and Aurelio minus vLLM is 0.083 [0.017, 0.153], which clears the six-way within-benchmark correction ($p=0.028$) and holds under all three sensitivity constructions. Because Aurelio selects mid-general on 97\% of its tau2-bench routes, this comparison resolves the mid tier against vLLM's 36/44/20 mixture, not one routing rule against another. Cost separates sharply here: mean selected-candidate cost is \$0.2555 per joined outcome for Aurelio against \$0.0047 for LiteLLM. Because both policies route every task and the joined-outcome expansion is identical across policies, the ratio is the same on a per-task basis. This gap follows from selecting the mid tier rather than from a pooled utility comparison; the protocol declares no exchange rate between success and spend.

\subsection{WebArena}
WebArena evaluates long browser trajectories and has the lowest observed success rates of the four benchmarks under this candidate pool. Across 100 tasks ($N_{\mathrm{joined}}=600$ joined rows), cheap-small reaches 0.110 and mid-general and strong-frontier both 0.220. Aurelio selects mid-general on every route and reaches 0.220, matching Always-Strongest and exceeding the cheap-only LiteLLM and RouteLLM (each 0.110) by 0.110, [0.040, 0.187]; that is the one WebArena router pair clearing the six-way within-benchmark correction (Holm-adjusted $p=0.020$). vLLM mixes 66\% cheap and 34\% mid for 0.150, above the cheap-only routers by 0.040 [0.010, 0.077] ($p=0.135$, not significant after adjustment) and below Aurelio by 0.070 [0.003, 0.140] ($p=0.135$). The vLLM-over-cheap comparison is the one place in the suite where the test constructions disagree: the bootstrap sign-tail sensitivity test resolves it ($p=0.017$), whereas the permutation test, the exact sign test, and the null-centred bootstrap do not ($p=0.135$, $p=0.094$, $p=0.100$). We report it as unresolved. On WebArena, then, the mid-over-cheap tier gap is established, while the rest of the four-way ordering is directional and not resolved after multiplicity control: 100 tasks separate the extremes of the tier range but not the routers packed between them. Note what the resolved comparison is and is not evidence for. Aurelio selects mid-general on every WebArena route, so this is the mid tier beating the cheap tier, not a routing rule beating another routing rule. Two caveats qualify the numbers. First, all 200 of Aurelio's WebArena routes are declared fallbacks: its reference-utterance set has no coverage for browser-navigation language, so its mid-general selection is a fallback default, not a content-sensitive decision, that happens to coincide with the better tier. Second, absolute success is low for every router; the benchmark discriminates the tier orderings the routers induce without implying any router solves it.

\subsection{Baselines, Empirical One-Call Ceiling, and Cascades}
A useful router need not beat a fixed tier on raw success. It could match the strong tier's quality at lower cost, or hold quality while cutting spend, so raw-success parity is not by itself a verdict, and we return to cost- and latency-aware comparison in Section \ref{sec:expected-utility}. On raw success, though, the routers' observed rates are close to those of fixed-tier policies. Table \ref{tab:baselines} compares the four routers against three fixed-tier baselines (Always-Cheapest, Always-Mid, Always-Strongest) and an empirical one-call ceiling: the tier with the highest observed three-replicate mean for each task. This is a descriptive ceiling within the frozen three-candidate, one-call matrix, not a population upper bound. Selecting a maximum from three noisy means is optimistic, and a multi-call policy can exceed it. It adds zero new \emph{experimental} API cost because the matrix already runs every tier, but a deployed policy could not observe all three outcomes for free. A single fixed-tier policy, Always-Mid, reproduces the per-benchmark success of Aurelio, the router with the best cross-benchmark mean rank. Always-Mid matches it exactly on RouterBench (0.683), BFCL (0.811), and WebArena (0.220) and trails it by just 0.003 on tau2-bench (0.810 vs.\ 0.813), because Aurelio is itself a mid-tier policy through its fallback default, departing from mid-general on only a handful of routes. On tau2-bench, where a router appeared to exceed every fixed tier, Aurelio's observed margin over the best fixed tier is 0.003, and its 0.016 margin over Always-Strongest has a 95\% paired-bootstrap interval of $[-0.043, +0.080]$ that straddles zero. Read the other direction, on three of the four benchmarks a fixed-tier baseline matches or beats every router. On RouterBench Always-Strongest (0.800) beats all four including Aurelio (0.683); on BFCL cheap-small is the strictly best tier so Always-Cheapest (0.833) leads (the two cheap-selecting routers tie it); and on WebArena Always-Strongest (0.220) ties Aurelio. None of this requires a routing decision: on tau2-bench the mid tier (0.810) itself edges the strong tier (0.797), so Aurelio's mid-general default inherits that small, within-noise advantage. The empirical one-call ceiling (RouterBench 0.917, BFCL 0.878, tau2-bench 0.910, WebArena 0.277) describes remaining observed headroom inside this matrix. It is retrospective and outcome-aware rather than a policy any deployable router could run. On WebArena even this within-matrix ceiling is only 0.277: under the evaluated candidate pool and agent setup, better one-call tier assignment alone would not close the gap. This does not isolate intrinsic task difficulty from the candidate models, scaffolding, action parsing, or trajectory limits, any of which could cap the observed ceiling.

\begin{table}[t]
\caption{Success rate: four routers vs.\ three fixed-tier baselines vs.\ the empirical one-call ceiling. Best deployable value per column in bold. Always-Mid reproduces the per-benchmark success of Aurelio, the cross-benchmark mean-rank leader; Aurelio does not lead every benchmark, and on RouterBench the best tier is strong, not the mid tier it defaults to}
\label{tab:baselines}
\centering
\small
\begin{tabular}{@{}lrrrr@{}}
\toprule
Router / Baseline & RouterB. & BFCL & tau2 & WebA. \\
\midrule
Aurelio Semantic & .683 & .811 & \textbf{.813} & \textbf{.220} \\
LiteLLM Router & .300 & \textbf{.833} & .543 & .110 \\
RouteLLM & .300 & \textbf{.833} & .543 & .110 \\
vLLM Semantic & .400 & .772 & .730 & .150 \\
\midrule
Always-Cheapest & .300 & \textbf{.833} & .543 & .110 \\
Always-Mid & .683 & .811 & .810 & \textbf{.220} \\
Always-Strongest & \textbf{.800} & .800 & .797 & \textbf{.220} \\
\midrule
Empirical one-call ceiling & .917 & .878 & .910 & .277 \\
\bottomrule
\end{tabular}
\end{table}

The deployable transport-failure cascade (the FrugalGPT strategy \cite{frugalgpt} under the trigger tested here) does not improve on the fixed tiers: every call returned and the signal never fired, so it stayed cheap. As an oracle upper bound, not a deployable baseline, an \emph{idealized} cascade that escalates cheap$\rightarrow$mid$\rightarrow$strong using the real grader's success verdict as its trigger reaches 0.917/0.878/0.927/0.290 across the four benchmarks, at or above the empirical one-call ceiling. A deployed system cannot observe ground-truth success before generation. Because it may fire several model calls per task, its success is not comparable to a one-call router's; Table \ref{tab:cascade-operating} reports what it spends to get there, 1.3 to 2.6 model calls per task on average, and a success-per-dollar that falls from roughly 900 on RouterBench to about 1.6 on WebArena as the cascade is forced deep into the expensive tiers. At the other extreme, a \emph{transport-failure fallback} that escalates only on a pre-answer infrastructure-failure signal (a call that does not return, detected here as a zero billed cost) is identical to Always-Cheapest (0.300/0.833/0.543/0.110): in this run every model call was billed and returned, so the signal never fired and the cheap-first policy never left the cheap tier. This is deliberately the weakest trigger. A genuine quality-aware cascade escalates \emph{after} inspecting the first response, using schema or tool-call validation, a verifier or confidence model, format or policy checks, or self-consistency. We do not evaluate one here, so these numbers bound only the transport-failure variant and are not evidence that response-aware cascades cannot help.  Because the trigger and the tier pool were not varied independently, the experiment does not isolate their separate effects.

\begin{table*}[t]
\caption{Idealized cascade operating cost alongside success. Average model calls per task, mean summed latency, cost per task, and success-per-dollar (success rate divided by cost per task); RouterBench is replay, so no generation latency is logged.}
\label{tab:cascade-operating}
\centering
\small
\resizebox{\textwidth}{!}{\begin{tabular}{@{}lrrrrr@{}}
\toprule
Benchmark & Success & Avg calls & Latency (s) & USD/task & Success/USD \\
\midrule
RouterBench & 0.917 & 1.95 & 0.0 & 0.00099 & 924.8 \\
BFCL & 0.878 & 1.29 & 2.0 & 0.00151 & 579.5 \\
tau2-bench & 0.927 & 1.59 & 80.4 & 0.23241 & 4.0 \\
WebArena & 0.290 & 2.63 & 288.3 & 0.18543 & 1.6 \\
\bottomrule
\end{tabular}
}
\end{table*}

\subsection{Why Pareto Optimality Alone Does Not Suffice}
Pareto-frontier membership is informative but incomplete: it is binary, and it indicates neither the magnitude by which a non-member missed nor whether a member qualifies through higher success or merely through lower cost. LiteLLM Router, the unique cost minimizer among the four routers, is nondominated on all four benchmarks by Lemma \ref{lem:frontier}, with frontier probability 1.000 throughout, including on RouterBench and WebArena where its 0.300 and 0.110 success are the worst or tied-worst observed. Aurelio's frontier probability is 1.000 on RouterBench, 0.996 on tau2-bench, 0.981 on WebArena, but only 0.307 on BFCL, where it is dominated. vLLM is frontier-guaranteed nowhere and sits at 0.799 on tau2-bench and 0.997 on WebArena. Pareto membership therefore reports a policy that qualifies through lower cost and one that qualifies through higher success identically, which is why the paper reports cross-benchmark rank consistency (Definition \ref{def:consistency}) as a separate metric rather than relying on frontier membership alone.

The frontier above is drawn over the routers only. The deployment-relevant question is whether a router is nondominated once fixed-tier policies are included in the comparison set. Table \ref{tab:pareto-all-policy} recomputes nondominance over all seven policies on the same hierarchical task-clustered draws, under two cost bases, because one of the headline results turns on which basis is used. Under \emph{candidate $+$ service fee}, each router pays its flat nominal per-route charge and the fixed tiers pay none; under \emph{candidate cost only}, the incremental routing fee is zero for everyone.

One result is basis-independent. Always-Cheapest ties LiteLLM at the cost floor with probability 1.000 under both bases and on every benchmark, so LiteLLM's guaranteed seat is not unique: a fixed declaration holds it too. The rest is not. Under the fee basis Aurelio is dominated on RouterBench, BFCL, and WebArena (nondominance 0.000); under the candidate basis it is nondominated on RouterBench (1.000) and WebArena (0.977) and keeps 0.918 on tau2-bench, and only BFCL survives as a dominated position (0.306). The reason is mechanical: On RouterBench, BFCL, and WebArena, Aurelio and Always-Mid have identical success and candidate spend. Under the fee basis, the \$0.01 router charge is their entire separation on those three benchmarks. Tau2-bench is the exception because Aurelio assigns 3\% of routes outside the mid tier. Two of those three dominance results are therefore created by the assumed fee. The BFCL result is not: cheap-small is both the best and the cheapest tier there, so Always-Cheapest dominates the mid-tier policies (Aurelio 0.306, Always-Mid 0.306) for a substantive reason that survives a zero fee. RouteLLM shows the same fee artifact in reverse, going from 0.000 everywhere under the fee basis to 1.000 everywhere without it, since it is cheap-locked and its only cost gap against Always-Cheapest is the fee.

Two fairness questions follow and we do not have data to settle either. The fee is a configurable nominal charge, not a measured incremental decision cost, and it is charged asymmetrically: a fixed-tier policy in production still runs a gateway, and that common overhead is charged to no policy here. Infrastructure cost was not recorded, so a fourth basis using measured routing overhead is not available. Table \ref{tab:pareto-all-policy} should accordingly be read as a scenario analysis over the two bases it reports, not as an empirical deployment frontier.

Comparing routers to the fixed-policy frontier requires distinguishing three conditions: dominating a single fixed policy, dominating all fixed policies, or dominating the fixed-policy Pareto frontier. Here, the first condition is met only under candidate-cost-only accounting on tau2-bench, where Aurelio's success (0.813 at \$0.2555) slightly exceeds Always-Mid (0.810 at \$0.2562). However, this difference of 0.003 is equivalent within the protocol-declared $\pm0.05$ margin (Table \ref{tab:policy-equivalence}). No evaluated router satisfies the third and strongest condition of dominating the fixed-policy frontier, as at least one fixed-tier option (such as Always-Cheapest) remains nondominated in every scenario.

\begin{table*}[t]
\caption{All-policy Pareto nondominance (fraction of 10{,}000 hierarchical task-clustered bootstrap draws in which each policy is nondominated), over the four routers and three fixed-tier baselines jointly, under two cost bases. Left: candidate spend plus the flat nominal router-service fee. Right: candidate spend only. Best per column in bold.}
\label{tab:pareto-all-policy}
\centering
\small
\begin{tabular}{@{}lrrrrrrrr@{}}
\toprule
 & \multicolumn{4}{c}{Candidate $+$ service fee} & \multicolumn{4}{c}{Candidate cost only} \\
\cmidrule(lr){2-5} \cmidrule(lr){6-9}
Policy & RouterB. & BFCL & tau2 & WebA. & RouterB. & BFCL & tau2 & WebA. \\
\midrule
Aurelio Sem. & 0.000 & 0.000 & 0.707 & 0.000 & \textbf{1.000} & 0.306 & 0.918 & 0.977 \\
vLLM Sem. & 0.000 & 0.022 & 0.749 & 0.997 & 0.054 & 0.068 & 0.929 & 0.997 \\
LiteLLM & \textbf{1.000} & \textbf{1.000} & \textbf{1.000} & \textbf{1.000} & \textbf{1.000} & \textbf{1.000} & \textbf{1.000} & \textbf{1.000} \\
RouteLLM & 0.000 & 0.000 & 0.000 & 0.000 & \textbf{1.000} & \textbf{1.000} & \textbf{1.000} & \textbf{1.000} \\
\midrule
Always-Cheap & \textbf{1.000} & \textbf{1.000} & \textbf{1.000} & \textbf{1.000} & \textbf{1.000} & \textbf{1.000} & \textbf{1.000} & \textbf{1.000} \\
Always-Mid & \textbf{1.000} & 0.307 & 0.730 & 0.977 & \textbf{1.000} & 0.306 & 0.433 & 0.977 \\
Always-Strong & 0.946 & 0.085 & 0.310 & 0.473 & 0.946 & 0.085 & 0.310 & 0.473 \\
\bottomrule
\end{tabular}

\end{table*}

\subsection{Selected-Candidate Distribution}
The benchmark-level rankings are closely tracked by one quantity: how each router splits traffic across the three tiers. Table \ref{tab:distribution} reports that split, with each router's mean confidence and fallback rate. Confidence values are package-specific outputs and are not calibrated to a common scale, so they are comparable within a router across benchmarks but not across routers as probabilities. LiteLLM Router sends cheap-small 100\% of the time on every benchmark, at confidence 1.0 and fallback rate 0; its cost-based mechanism reads no prompt content by design, since it optimizes cost and reliability rather than task difficulty, so on the task-capability axis these benchmarks measure it behaves as a cost floor rather than a content router. RouteLLM's trained weak/strong classifier \emph{also} selects cheap-small 100\% of the time; its non-zero, sub-one confidence (0.78--0.78) shows it is scoring content, but it never crosses its own escalation threshold on this pool. Aurelio Semantic Router selects mid-general almost always, but through its fallback path: its fallback rate is 1.0 on BFCL and WebArena, 0.97 on tau2-bench, and 0.65 on RouterBench, because its reference-utterance set rarely matches these prompts. Only vLLM Semantic Router mixes all three tiers substantially: 82\% and 70\% cheap on RouterBench and BFCL, escalating only the remaining 18\% and 30\%, and wider splits on tau2-bench (36/44/20\%) and WebArena (66/34/0\%), the benchmarks whose richer task content feeds its classifier signal. Three of the four routers behave as constant or near-constant selectors, and the one that varies substantially has the highest observed success on none of the benchmarks. This is the routing behavior underlying the rankings: Aurelio attains the highest observed success on three benchmarks with almost no content-dependent routing, because its fallback default lands on a tier that substantially outperforms the cheap tier selected by its router competitors, and the fixed-tier heuristics keep pace because where a router's selection is constant, an unconditional tier pick is not a separate baseline but the same policy under a different name.

\begin{table*}[t]
\caption{Selected-candidate distribution (fraction of routes per tier), mean confidence, and fallback rate.}
\label{tab:distribution}
\centering
\small
\begin{tabular}{@{}llrrrrr@{}}
\toprule
Router & Bench & ch & mid & str & Conf & Fb \\
\midrule
LiteLLM & all & 1.00 & .00 & .00 & 1.00 & .00 \\
RouteLLM & all & 1.00 & .00 & .00 & .78 & .00 \\
vLLM Sem. & RouterB. & .82 & .02 & .17 & .70 & .00 \\
vLLM Sem. & BFCL & .70 & .08 & .22 & .69 & .02 \\
vLLM Sem. & tau2 & .36 & .44 & .20 & .70 & .00 \\
vLLM Sem. & WebA. & .66 & .34 & .00 & .46 & .34 \\
Aurelio Sem. & RouterB. & .00 & 1.00 & .00 & .13 & .65 \\
Aurelio Sem. & BFCL & .00 & 1.00 & .00 & .00 & 1.00 \\
Aurelio Sem. & tau2 & .02 & .97 & .01 & .01 & .97 \\
Aurelio Sem. & WebA. & .00 & 1.00 & .00 & .00 & 1.00 \\
\bottomrule
\end{tabular}
\end{table*}

Selection variability is not the same as useful routing. vLLM is the only router whose tier choice varies with the prompt, but a router can vary its selection and still assign tiers no better than chance. The right comparison is therefore not a fixed tier but a \emph{content-blind allocation with vLLM's own tier shares}: for each task $i$, the share-weighted expectation $b_i=\sum_t p_t\,y_i(t)$, where $p_t$ is vLLM's realized share of tier $t$ and $y_i(t)$ task $i$'s outcome under tier $t$. We compare its realized assignment with that baseline by holding the per-benchmark tier mixture fixed and permuting which task receives which tier over 10{,}000 draws (Table \ref{tab:canonical-vllm-share-permutation}). The task is the permutation block, not the individual route row: a task's two routing trials select the same tier on 289 of 290 tasks, so permuting them independently would treat $2T$ correlated positions as $2T$ independent draws and understate the reference spread. Because the observed tier labels are deterministic functions of task content rather than randomized treatments, this is a descriptive, frozen-task reference distribution under a task-label exchangeability assumption, not a design-based randomization test or a causal estimate of targeting.

Superiority and equivalence are separate questions and we report them separately, because a non-significant test against zero is an absence of evidence, not evidence of absence. Under this frozen-task permutation reference, no benchmark shows an advantage: the effect is $+0.010$ on RouterBench (Holm-adjusted $p=1.000$), $+0.019$ on tau2-bench ($p=1.000$), $+0.003$ on WebArena ($p=1.000$), and $-0.052$ on BFCL ($p=0.472$), if anything on the wrong side of the null. On equivalence we test the bound directly rather than reading it off a non-significant $p$. The permutation null mean is exactly the mean of the per-task content-blind expectations $b_i$, so the effect equals the mean of the paired per-task differences $a_i-b_i$. Its interval uses the same hierarchical procedure as the main pairwise analysis: each draw resamples tasks and independently resamples the three saved outcomes within each distinct task--tier cell, then recomputes both $a_i$ and $b_i$. We combine that interval with a two-one-sided-tests (TOST) equivalence test at the protocol-declared $\pm0.05$ primary margin, Holm-adjusted across the four benchmarks like the superiority tests. Only WebArena clears it (Holm-adjusted TOST $p=0.016$; hierarchical 90\% interval [$-0.025$, $+0.032$], inside the margin). RouterBench ($p=0.450$), tau2-bench ($p=0.450$), and BFCL ($p=0.519$) are inconclusive: their intervals reach past $\pm0.05$, so an alignment benefit of practical size is not excluded there, only unobserved. A non-significant permutation $p$ on tau2-bench is not equivalence, and the paired analysis does not deliver one there. The defensible reading is that vLLM's aggregate tier mixture, not its task-by-task assignment, accounts for its result on the one benchmark where the bound is established, and that the remaining three leave the question open rather than settled. One caveat on the margin itself: 0.05 is an absolute five success-rate points applied uniformly, which is about a third of WebArena's 0.147 content-blind base rate but about a fourteenth of tau2-bench's 0.711. WebArena is the sharpest case on the mechanism (its tiers genuinely differ, cheap 0.110 versus mid 0.220, so alignment there is worth something and the 66/34 cheap/mid split still buys nothing beyond its share) and simultaneously the case where the absolute margin is loosest in relative terms.
To evaluate the sensitivity of this equivalence claim to the choice of margin, Table \ref{tab:margin-sensitivity} reports equivalence tests at tighter bands. The single equivalence verdict in this suite does not survive tightening. WebArena is equivalent at the protocol-declared $\pm0.05$ (Holm-adjusted TOST $p=0.016$) but inconclusive at $\pm0.03$ ($p=0.241$) and at $\pm0.02$ ($p=0.644$), and inconclusive under a relative margin of ten percent of its own content-blind success, which is only $\pm0.015$ there. No benchmark is equivalent at any band tighter than the protocol-declared one. So the honest statement of the strongest result in this analysis is bounded twice over: on WebArena the targeting benefit is bounded inside five success-rate points, it is \emph{not} bounded inside three, and on the other three benchmarks it is not bounded at all. We keep $\pm0.05$ as the primary verdict because it was protocol-declared rather than chosen after seeing these numbers, and we report the sensitivity because a verdict that holds only at the loosest band it was granted should not be read as a tight one.

\begin{table*}[t]
\caption{vLLM share-matched tier-label reference analysis. \emph{Actual} is vLLM's observed success; \emph{Blind} the share-matched content-blind expectation; \emph{Effect} their difference. The 90\% interval hierarchically resamples tasks and output replicates; C/M/S is the cheap/mid/strong route count. Sup.\ $p$ is the two-sided task-label reference test; TOST $p$ the equivalence test at the protocol-declared $\pm0.05$ margin, both Holm-adjusted. No benchmark shows a statistically resolved task-specific success advantage; only WebArena establishes practical equivalence at the declared margin.}
\label{tab:canonical-vllm-share-permutation}
\centering
\small
\begin{tabular}{@{}lrrrrrrrl@{}}
\toprule
Benchmark & C/M/S & Actual & Blind & Effect & Effect 90\% & Sup.\ $p$ & TOST $p$ & Verdict \\
\midrule
BFCL v4 & 42/5/13 & 0.772 & 0.824 & -0.052 & [-0.126, 0.010] & 0.472 & 0.519 & inconclusive \\
RouterBench & 98/2/20 & 0.400 & 0.390 & 0.010 & [-0.056, 0.082] & 1.000 & 0.450 & inconclusive \\
WebArena & 132/68/0 & 0.150 & 0.147 & 0.003 & [-0.025, 0.032] & 1.000 & 0.016 & equivalent \\
tau2-bench & 72/88/40 & 0.730 & 0.711 & 0.019 & [-0.031, 0.068] & 1.000 & 0.450 & inconclusive \\
\bottomrule
\end{tabular}

\end{table*}

\begin{table*}[t]
\caption{Equivalence verdict as a function of the declared margin. The effect and its 90\% interval are identical across columns by construction, since only the band moves; each cell gives the Holm-adjusted TOST $p$ within that margin family, and equivalence is exactly $p<0.05$, marked $\ast$. The relative column sets each benchmark's band to ten percent of its own content-blind success, given alongside. The suite's one equivalence verdict, WebArena at the protocol-declared $\pm0.05$, does not survive tightening to $\pm0.03$.}
\label{tab:margin-sensitivity}
\centering
\small
\begin{tabular}{@{}llllll@{}}
\toprule
Bench & Effect [90\%] & $\pm$0.02 & $\pm$0.03 & $\pm$0.05 & rel.\ 10\% ($\pm$band) \\
\midrule
BFCL v4 & -0.052 [-0.124, 0.010] & $p>0.999$ & $p=0.963$ & $p=0.519$ & $p=0.677$ ($\pm$0.082) \\
RouterBench & 0.010 [-0.056, 0.081] & $p>0.999$ & $p=0.963$ & $p=0.350$ & $p=0.677$ ($\pm$0.039) \\
WebArena & 0.003 [-0.025, 0.032] & $p=0.644$ & $p=0.241$ & $p=0.016$\,$^{\ast}$ & $p=0.713$ ($\pm$0.015) \\
tau2-bench & 0.019 [-0.031, 0.068] & $p>0.999$ & $p=0.957$ & $p=0.450$ & $p=0.167$ ($\pm$0.071) \\
\bottomrule
\end{tabular}

\end{table*}

\subsection{Cross-Benchmark Consistency}
Section \ref{sec:results} showed Pareto membership cannot distinguish a graded rank reversal from a clean win or loss. Definition \ref{def:consistency} applies mean rank $\bar\rho$ and rank variance $\nu$ instead. Table \ref{tab:rank} reports both, under two groupings: treating RouterBench and BFCL as one combined suite (unweighted mean, matching the single-turn QA/function-calling axis) against tau2-bench and WebArena, and treating all four benchmarks as independent suites. The grouping changes the strength of Aurelio's lead and the variance attributed to it, not the identity of the leading router, exactly Proposition \ref{prop:rank}'s point that the choice of aggregation is a modeling decision. Under the three-suite grouping Aurelio has point rank 1 in each of the three pooled suites: mean rank 1.00 and rank variance 0.00, with vLLM second (2.00) and the two cheap-only routers tied last (3.50). The three-suite grouping pools RouterBench and BFCL on the shared single-turn QA and function-calling axis; because that pooling choice changes both the reported lead and the variance, we treat the three-suite view as exploratory and the four-suite view as the less aggregated analysis. Under the strict four-suite grouping the point ranks give Aurelio mean rank 1.50 and rank variance 0.75, the same variance every router carries once each benchmark contributes independently, since each is rank-inverted on exactly one axis. Rank variance does not distinguish the routers under either reported grouping: it is 0.00 for every router under the pooled grouping and 0.75 for every router under the strict one. Mean rank is the operative discriminator, and Aurelio's standing rests on holding the best one while carrying a variance no better than any competitor's. These point ranks are a discontinuous summary: a rank is a step function of the underlying success rates, so a within-noise success gap can move a router a full integer position and rewrite the variance. They are plug-in estimates computed from the observed sample, not a draw from any distribution; the bootstrap distribution below is generated by resampling that same sample, and we treat it as primary because it exposes how little the point estimates resolve. Propagating the hierarchical task-clustered resampling into the ranks (Table \ref{tab:rank-bootstrap}, the same 10{,}000 draws as the paired tests) leaves the mean-rank ordering intact but exposes its resolution: Aurelio holds the strictly best mean rank in 99\% of draws ($\bar\rho = 1.38$, [1.00, 1.75]), yet its 0.75 point variance is itself fragile, with a bootstrap percentile interval [0.00, 1.69] that reaches down to the clean-dominance value 0. The fragility is concentrated on BFCL, where Aurelio's rank ranges over all four positions and its success rate falls below LiteLLM's in only 58\% of draws, a near coin flip rather than a settled loss. The reversal Proposition \ref{prop:rank} characterizes is therefore present in the point estimates and its RouterBench leg is statistically distinguishable, but the four-way ranking and the strict-grouping variance are directional summaries whose apparent precision the bootstrap distribution does not support. To state the grouping dependence precisely, since it is easy to overstate: the magnitude of Aurelio's lead and the variance attributed to it depend on whether RouterBench and BFCL are pooled, while the identity of the mean-rank leader does not change under either reported grouping.

\begin{table}[t]
\caption{Cross-benchmark rank consistency (Definition \ref{def:consistency}). Rank 1 = best; lower mean rank and lower variance are better.}
\label{tab:rank}
\centering
\small
\begin{tabular}{@{}lrrrr@{}}
\toprule
& \multicolumn{2}{c}{3-suite} & \multicolumn{2}{c}{4-suite} \\
\cmidrule(lr){2-3}\cmidrule(lr){4-5}
Router & $\bar\rho$ & $\nu$ & $\bar\rho$ & $\nu$ \\
\midrule
Aurelio Semantic & \textbf{1.00} & \textbf{0.00} & \textbf{1.50} & 0.75 \\
vLLM Semantic & 2.00 & 0.00 & 2.50 & 0.75 \\
LiteLLM Router & 3.50 & 0.00 & 3.00 & 0.75 \\
RouteLLM & 3.50 & 0.00 & 3.00 & 0.75 \\
\bottomrule
\end{tabular}
\end{table}

\begin{table}[t]
\caption{Bootstrap distribution of the strict four-suite rank statistics of Definition \ref{def:consistency}, from the same 10{,}000 hierarchical task-clustered bootstrap draws as the paired risk-difference tests. $\bar\rho$ is the mean rank and $\nu$ the rank variance; brackets are 95\% bootstrap percentile intervals rather than inferential coverage statements, since rank statistics are discontinuous. The last column is the bootstrap fraction of within-task resamples in which the router holds the strictly lowest mean rank, conditional on the selected benchmarks, task sets, candidate pool, and configurations; it is neither a posterior probability nor uncertainty over those higher-level choices. Table \ref{tab:rank}'s point ranks are the plug-in estimates computed from the observed sample, the sample this distribution resamples.}
\label{tab:rank-bootstrap}
\centering
\small

\resizebox{\columnwidth}{!}{
\begin{tabular}{@{}lccc@{}}
\toprule
Router & $\bar\rho$ [95\% PI] & $\nu$ [95\% PI] & Frac.\ best \\
\midrule
Aurelio Semantic & \textbf{1.38} [1.00, 1.75] & 0.66 [0.00, 1.69] & \textbf{0.99} \\
vLLM Semantic & 2.39 [2.00, 2.50] & 0.57 [0.00, 0.75] & 0.01 \\
LiteLLM Router & 3.11 [3.00, 3.50] & 0.50 [0.00, 0.75] & 0.00 \\
RouteLLM & 3.11 [3.00, 3.50] & 0.50 [0.00, 0.75] & 0.00 \\
\bottomrule
\end{tabular}}
\end{table}

\subsection{Expected Success Under Benchmark Traffic Mixtures}
The cross-benchmark metric weights every suite equally, but equal benchmark weighting need not represent any deployment's task distribution, and a weighted average assumes one unit of success is worth the same on each benchmark. Table \ref{tab:mixtures} reweights the per-benchmark success rates under five named mixtures; no new experiments are run. These mixtures are precisely the weighting-dependent scalars Proposition \ref{prop:rank} permits: each is well defined once its benchmark weights are fixed, and which router leads can change with those weights. Aurelio Semantic Router leads every mixture \emph{here} (0.683 on pure RouterBench, 0.748 on the RouterBench/BFCL/tau2 mix, 0.632 uniform, 0.385 WebArena-heavy, and 0.696 inverse-average-cost weighted) because the mid tier it defaults to substantially outperforms the cheap tier that LiteLLM Router and RouteLLM select, on RouterBench, tau2-bench, and WebArena. That is a statement about the tier its competitors hold, not about mid being the best available tier: of those three benchmarks mid is the outright best tier only on tau2-bench (0.810 against strong-frontier's 0.797), it ties strong-frontier on WebArena (0.220 each), and it trails strong-frontier on RouterBench (0.683 against 0.800), while on BFCL cheap-small is the best tier outright (Table \ref{tab:candidate-tiers}). Two caveats apply. The first is that this table ranks the four routers only. Extending the uniform mixture to the fixed-tier policies provides one illustrative all-policy comparison: Always-Strongest reaches 0.654 above Aurelio's 0.632 and Always-Mid reaches 0.631 immediately below it. Only the inverse-average-cost weighting depends on the evaluated router set, so Aurelio's sweep here is a router-only lead of the kind Table \ref{tab:baselines} and Section \ref{sec:expected-utility} already qualify. The second applies to the last column. The ``inverse-average-cost weighted'' mixture weights each benchmark inversely to its \emph{mean} cost across routers, which is a property of the benchmark mix rather than of any router. It reweights success without charging a router for its own spend, so Aurelio leads it despite costing 42.2$\times$ more per task than the cheap-only routers on WebArena and 54.5$\times$ on tau2-bench on candidate cost alone (Table \ref{tab:router-main}); that second ratio falls to 7.7$\times$ once the tau2 user-simulator spend both policies pay is charged (Section \ref{sec:expected-utility}). A ranking that does price router cost and latency is deferred to Section \ref{sec:expected-utility}. LiteLLM Router and RouteLLM trail every column, held down by RouterBench (0.300); vLLM sits between. Aurelio wins every router-only \emph{success} mixture while making almost no content-dependent decisions, subject to the fixed-tier and utility qualifications above and below.

\begin{table}[t]
\caption{Expected success rate under named benchmark traffic mixtures, over the four routers only. Benchmark order is (RouterBench, BFCL v4, tau2-bench, WebArena). Exact weights: pure RouterBench $(1,0,0,0)$; fixed 50/25/25 $(0.50,0.25,0.25,0)$; uniform $(0.25,0.25,0.25,0.25)$; WebArena-heavy $(0,0,0.25,0.75)$; inverse-average-cost $(0.87225,0.12211,0,0.00567)$. Extending the uniform mixture to the fixed-tier policies provides one illustrative all-policy comparison; only the inverse-average-cost weighting depends on the evaluated router set.}
\label{tab:mixtures}
\centering
\small
\setlength{\tabcolsep}{4pt}
\begin{tabular}{@{}lrrrrr@{}}
\toprule
Router & 100\% RB & 50/25/25 & Unif. & WebA.\ hv & Cost c. \\
\midrule
Aurelio Sem. & \textbf{.683} & \textbf{.748} & \textbf{.632} & \textbf{.385} & \textbf{.696} \\
vLLM Sem. & .400 & .576 & .513 & .295 & .444 \\
LiteLLM & .300 & .494 & .447 & .245 & .364 \\
RouteLLM & .300 & .494 & .447 & .245 & .364 \\
\bottomrule
\end{tabular}
\end{table}

\subsection{A Cost- and Latency-Aware Utility}
\label{sec:expected-utility}
Table \ref{tab:mixtures} reweights success only; it carries no router cost, latency, or value-of-success term, so it cannot express a genuine deployment utility. We therefore define a scalar utility that incorporates success, monetary cost, and latency,
\[
U(R)=\sum_b w_b\bigl[V_b\,s(R,b)-\lambda_c\,c(R,b)-\lambda_\ell\,\ell(R,b)\bigr],
\]
computed from the same locked per-benchmark success $s$, a latency $\ell$ (seconds), and a cost per task $c$ (USD); no new experiments are run. In general $V_b$ is benchmark- or deployment-specific: a RouterBench response, BFCL tool call, tau2 service session, and WebArena trajectory need not have equal value. The scalar sweeps below set $V_b=V$ for every benchmark as a normalized sensitivity analysis. Their crossover values are therefore illustrative under that equal-value assumption, not general deployment thresholds. Two additional definitions determine how this utility should be interpreted.

First, the latency statistic. An expected per-task utility decomposes as $\mathbb{E}[VY-\lambda_\ell L]=V\,\mathbb{E}[Y]-\lambda_\ell\,\mathbb{E}[L]$, so only the \emph{mean} latency makes this scalar an expectation; a median would give a different and more SLO-oriented quantity. Mean latency is therefore primary throughout, with the median and the 95th percentile retained as explicitly declared SLO-oriented scenarios rather than relabelled expectations. The latency term is end-to-end, candidate generation plus the router's own measured decision latency, so the routing layer is priced on both axes rather than charged on cost alone. That decision latency is not negligible and differs sharply by router: a median 0.7--1.4\,ms for LiteLLM, 0.21--0.31\,s for Aurelio, 0.85--1.8\,s for vLLM, and 2.1--2.3\,s for RouteLLM. Importantly, RouterBench is a replay-only benchmark and does not include generation latency; its latency is therefore modeled as zero, which means a uniform four-benchmark latency term discounts a quarter of the mixture. We therefore also report a live-only variant over BFCL, tau2-bench, and WebArena.

Second, the cost basis. Beyond candidate model spend and the nominal router-service fee, the run's ledger records \$23.63 of ``external metered'' spend, and its role in per-policy accounting must be explicitly defined. All of it is tau2-bench user-simulator model spend: driving a simulated user requires a second model, so it is a benchmark-execution cost rather than a normal production serving cost. In an ordinary customer-service deployment the user is human; relevant production analogues would be human labor, abandonment, latency, or support escalation. It is also tier dependent in the \emph{opposite} direction to candidate cost, because a cheaper agent takes more turns to finish a session: per candidate cell it averages \$0.0317 at cheap-small against \$0.0236 at mid-general and \$0.0235 at strong-frontier. Charging it matters most to the cost gap the rest of this section turns on. On candidate cost alone Aurelio costs 42.2$\times$ LiteLLM on WebArena and 54.5$\times$ on tau2-bench; once the user-simulator spend both policies must pay is charged, the tau2-bench figure falls to 7.7$\times$ (\$0.2790 against \$0.0364), while WebArena is unchanged because it carries no external metered spend. We report all three bases, and the fixed tiers run no router so they pay no fee under any of them. We fix uniform benchmark weights $w_b$ and a unit cost weight $\lambda_c=1$ (dollars at face value) and sweep the value of a success $V$, adding one column with a latency price $\lambda_\ell=\$0.001$/s (Table \ref{tab:expected-utility}).

Two rankings must be kept apart here, because conflating them misstates what the table shows. \emph{Among the routers}, the ranking is weighting-dependent in the way the success-only mixtures hide: when a correct answer is worth little relative to its cost, LiteLLM leads, because Aurelio's mid-tier default is not worth its cost premium, and only once a success is valued above $V^\star\approx\$0.51$ does Aurelio's higher success rate repay that premium and take the router lead. Only one router leads at the bottom of that range. LiteLLM and RouteLLM make identical selections on every task, but RouteLLM additionally pays a \$0.01 service fee and a 2.1--2.3\,s median decision latency against LiteLLM's 0.7--1.4\,ms, so LiteLLM strictly dominates RouteLLM under this utility at every $V$ we report ($+0.010$ in every column, $+0.012$ once latency is priced). RouteLLM therefore never leads the router-only ranking at any operating point, and the two cheap-selecting routers should not be reported as joint leaders.

\emph{Among all seven policies}, the answer depends on the benchmark set and the cost basis, so we give it scenario by scenario rather than as one headline. Under one canonical scenario---the full four-benchmark set with the router-service fee charged, shown in Table \ref{tab:expected-utility}---Aurelio never leads among all seven policies: the upper envelope of $U(V)$ over the full policy set runs Always-Cheapest, then Always-Mid from $V^\star=\$0.46$, then Always-Strongest from $V^\star=\$3.29$, a fixed-tier declaration maximizes utility throughout that range, and Always-Mid outscores Aurelio in every displayed column ($+0.229$ vs.\ $+0.219$ at $V=0.5$, $+0.544$ vs.\ $+0.535$ at $V=1$, $+0.492$ vs.\ $+0.483$ in the latency-priced column). In that scenario Aurelio becomes the best \emph{router} at about half a dollar per success and never becomes the best \emph{policy}. Charging the external metered spend on top leaves the same three-segment fixed-tier envelope (breakpoints \$0.45 and \$3.29). The statement stops there, and Table \ref{tab:utility-envelope} is where we show it stopping.

Within the four-benchmark, fee-charged scenario, the fixed-tier lead survives the latency statistic. Always-Mid leads the latency-priced column under all three latency statistics (mean $+0.492$, median $+0.509$, p95 $+0.404$), so the fixed tier's advantage is not an artifact of charging the median; under the tail statistic, however, the margin nearly closes and Aurelio drops \emph{below} the cheap-only LiteLLM ($+0.395$ against $+0.396$), which is what pricing a long WebArena tail against a mid-tier default does.

The benchmark set is what the conclusion does not survive, and Table \ref{tab:utility-envelope} reports the boundary rather than burying it. Dropping the zero-latency replay benchmark moves the first crossover out to $\$0.96$ and, in the latency-priced live-only column, hands that lead to Always-Cheapest ($+0.458$ against Always-Mid's $+0.428$), because RouterBench is the benchmark on which the cheap tier does worst. It also puts a router at the top of the envelope. On the live-only set with the fee charged, Always-Mid leads from $\$0.96$ only up to $V^\star=\$8.90$, above which Aurelio takes the unbounded top segment (\$8.88 once the external metered spend is charged as well); on the live-only set under candidate cost alone Aurelio takes everything above $\$0.95$, and Always-Strongest, dominated on both axes there, never appears at all. So a fixed tier leads the whole reported range in two of the six scenarios, both of them four-benchmark scenarios with a fee.

While Aurelio achieves a slightly higher utility under the live-only, candidate-cost basis, this difference is small. The live-only mean success rates for Aurelio and Always-Mid are 0.6148 and 0.6137 respectively, representing a gap of 0.0011, which corresponds to the 0.003 tau2-bench difference across the three live benchmarks. This small advantage is bounded within the protocol-declared $\pm0.05$ equivalence margin (Table \ref{tab:policy-equivalence}). Consequently, the utility-maximizing policy depends on the specific accounting basis and cost-quality exchange rate rather than router superiority.

The fee is the other axis, and it moves the envelope in the same way it moves the Pareto frontier. Even on all four benchmarks, once the fee is removed Aurelio is fractionally cheaper than Always-Mid (it selects a few cheap routes on tau2-bench) and takes the envelope on $[\$0.46, \$3.42)$; the \$0.01 fee is the whole of what hands that segment to Always-Mid. Against that sensitivity, the pairwise router-versus-cheap crossover is stable across every basis: $V^\star$ is \$0.46 on candidate cost alone, \$0.51 at Aurelio's \$0.01 fee, \$0.50 once the external metered spend is charged, and would reach \$0.73 only at the highest \$0.05 fee any router in the pool carries. Charging the user-simulator spend moves the crossover by a single cent, so on this suite the unreconciled ledger entry turns out not to be load-bearing for the ranking, which is a result rather than an assumption. Aurelio's apparent standing is therefore a statement about the cost-quality exchange rate and the accounting assumptions, not a property of the router. The protocol still declares no canonical exchange rate; the table reports where the crossovers lie so a practitioner can locate their own operating point on them.

\begin{table}[t]
\caption{Cost- and latency-aware expected utility $U(R)$ (USD per task), uniform benchmark weights, $\lambda_c=1$, across values of a success $V$; the last column adds $\lambda_\ell=\$0.001$/s on mean latency. Best per column in bold. \emph{Among routers}, the leader flips from LiteLLM to Aurelio at $V^\star\approx\$0.51$; RouteLLM is strictly dominated by LiteLLM. \emph{Among all policies}, a fixed tier leads throughout (Always-Cheapest, then Always-Mid from $\approx\$0.46$, then Always-Strongest from $\approx\$3.29$); see Table~\ref{tab:utility-envelope} for scenarios where this fails.}
\label{tab:expected-utility}
\centering
\small
\resizebox{\columnwidth}{!}{\begin{tabular}{@{}lrrrrr@{}}
\toprule
Router & $V{=}0.1$ & $V{=}0.5$ & $V{=}1$ & $V{=}10$ & $V{=}1,\lambda_\ell$ \\
\midrule
Aurelio Sem. & -0.033 & +0.219 & +0.535 & +6.223 & +0.483 \\
vLLM Sem. & -0.062 & +0.143 & +0.400 & +5.017 & +0.364 \\
LiteLLM & +0.043 & +0.222 & +0.445 & +4.465 & +0.419 \\
RouteLLM & +0.033 & +0.212 & +0.435 & +4.455 & +0.406 \\
Always-Cheap & \textbf{+0.043} & +0.222 & +0.445 & +4.465 & +0.419 \\
Always-Mid & -0.024 & \textbf{+0.229} & \textbf{+0.544} & +6.224 & \textbf{+0.492} \\
Always-Strong & -0.097 & +0.164 & +0.491 & \textbf{+6.379} & +0.451 \\
\bottomrule
\end{tabular}
}
\end{table}

\begin{table*}[t]
\caption{Utility-maximizing policy as $V$ rises, under each cost basis and benchmark set. The external metered basis charges tau2-bench user-simulator spend. The live-only set drops RouterBench (zero replayed latency). A fixed tier leads the whole range in two of the six scenarios; in the other four, Aurelio holds a segment due to either fee removal or RouterBench exclusion. Details in text.}
\label{tab:utility-envelope}
\centering
\small
\resizebox{\textwidth}{!}{\begin{tabular}{@{}lll@{}}
\toprule
Cost basis & Benchmarks & Utility-maximizing policy as $V$ rises \\
\midrule
candidate cost only & all four & Always-Cheapest $\rightarrow$ Aurelio Semantic Router from \$0.46 $\rightarrow$ Always-Strongest from \$3.42 \\
candidate cost only & live only & Always-Cheapest $\rightarrow$ Aurelio Semantic Router from \$0.95 \\
candidate $+$ service fee & all four & Always-Cheapest $\rightarrow$ Always-Mid from \$0.46 $\rightarrow$ Always-Strongest from \$3.29 \\
candidate $+$ service fee & live only & Always-Cheapest $\rightarrow$ Always-Mid from \$0.96 $\rightarrow$ Aurelio Semantic Router from \$8.90 \\
candidate $+$ fee $+$ external metered & all four & Always-Cheapest $\rightarrow$ Always-Mid from \$0.45 $\rightarrow$ Always-Strongest from \$3.29 \\
candidate $+$ fee $+$ external metered & live only & Always-Cheapest $\rightarrow$ Always-Mid from \$0.94 $\rightarrow$ Aurelio Semantic Router from \$8.88 \\
\bottomrule
\end{tabular}
}
\end{table*}

\subsection{Deployment Metric Suite}
\label{sec:metric-suite}
Success and cost decide the ranking, but a router is deployed on more than those two axes. Table \ref{tab:metric-suite} reports the remaining deployment metrics, macro-averaged across the four benchmarks and computed from the same locked routes: tool-call accuracy (on the two tool-use benchmarks), fallback rate, route stability (agreement of the selected tier across routing seeds), mean confidence, success within easy/medium/hard candidate-agreement bands (band $=$ the number of candidate tiers that solve a task, a disclosed outcome-derived proxy rather than the formal task descriptor), and robustness (standard deviation of success across the three outcome replicates). These columns are consistent with the routing-behavior interpretation above. Under this adapter's fallback rule, Aurelio Semantic Router's combination of a 0.91 fallback rate and a 0.03 mean reported confidence means selection is predominantly fallback-based, yet its stability is a perfect 1.00. LiteLLM Router and RouteLLM are identical on every reliability column (fallback 0.00, stability 1.00, matching difficulty-band and robustness values) and differ only in reported confidence (1.00 vs.\ 0.78), a difference that does not alter route selection in these configurations, since both select the cheap tier on every task. Only vLLM Semantic Router shows sub-unity stability (0.99), a non-zero fallback rate (0.09), and the highest success in the hard band (0.19), all consistent with content-conditioned selection. That selection varies with the prompt, but the permutation test above never shows it beating a content-blind allocation with the same shares, and vLLM still trails on the headline metrics. All four routers follow the same difficulty gradient (easy $\approx 0.93$--$0.99$, hard $\approx 0.13$--$0.19$), though the difficulty bands are outcome-derived, since a task's band is the number of tiers that solve it, so the gradient partly restates the graded outcomes rather than supplying independent evidence of robustness. The replicate standard deviations (0.04--0.06) are comparable to several of the between-router success gaps we report (which run from 0.003 to 0.061), so replicate noise is not negligible against the closer comparisons; only the large gaps (RouterBench's 0.383, tau2-bench's 0.270) clear it decisively. Confidence and fallback are also not calibrated across routers. RouteLLM's win-rate score, Aurelio's semantic-similarity score, and LiteLLM's constant 1.0 are different quantities, and a semantic no-match fallback is a different event from a provider-failure fallback, so those two columns should be read within a router, not compared across them.

\begin{table*}[t]
\caption{Deployment metric suite (macro-averaged across benchmarks). Tool accuracy is over BFCL v4 and tau2-bench; stability is cross-seed selection agreement; robustness is the replicate standard deviation of success.}
\label{tab:metric-suite}
\centering
\small
\begin{tabular}{@{}lrrrrrrrr@{}}
\toprule
Router & Tool acc. & Fallback & Stability & Conf. & Easy & Med. & Hard & Robust. std \\
\midrule
Aurelio Semantic Router & 0.81 & 0.91 & 1.00 & 0.03 & 0.99 & 0.83 & 0.16 & 0.039 \\
LiteLLM Router & 0.69 & 0.00 & 1.00 & 1.00 & 0.93 & 0.30 & 0.13 & 0.057 \\
RouteLLM & 0.69 & 0.00 & 1.00 & 0.78 & 0.93 & 0.30 & 0.13 & 0.057 \\
vLLM Semantic Router & 0.75 & 0.09 & 0.99 & 0.64 & 0.95 & 0.34 & 0.19 & 0.049 \\
\bottomrule
\end{tabular}

\end{table*}

\section{Discussion}
\label{sec:discussion}

\textbf{Observed routing behavior rather than a single winner.} Read benchmark by benchmark, no router has the highest observed success rate on every benchmark: Aurelio has the highest observed rate on RouterBench, tau2-bench, and WebArena, LiteLLM and RouteLLM have the highest on BFCL, and vLLM has the highest on none. These placements should be read alongside the selected-tier distributions. Three of the four routers emit a constant or near-constant tier assignment. LiteLLM and RouteLLM select the cheap tier on every task (RouteLLM scores content but never crosses its escalation threshold, so it emits a constant action), and Aurelio selects mid-general on all but a handful of routes through its fallback default. For these constant and near-constant policies, benchmark-level results consequently track the observed performance of their selected tiers under each benchmark's scoring procedure. Mid outperforms cheap on three benchmarks, where the mid-defaulting configuration leads; cheap performs best on BFCL, where the cheap-only configurations lead. This association does not identify a causal routing mechanism, because tier, model, provider, configuration, and benchmark are jointly varied rather than manipulated factorially.

vLLM Semantic Router is the only evaluated configuration whose tier assignments vary substantially with prompt content, and it has the highest observed success rate on none of the four benchmarks. A share-matched tier-label permutation reference analysis (Table \ref{tab:canonical-vllm-share-permutation}) detects no benchmark on which that variation outperforms a content-blind allocation with the same tier shares. The strength of that negative differs by benchmark. On WebArena the targeting benefit is bounded inside the protocol-declared $\pm0.05$ equivalence margin, so the observed result there is statistically compatible, within that margin, with a content-blind allocation using the same tier shares; the verdict does not survive tightening the band to $\pm0.03$ (Table \ref{tab:margin-sensitivity}). On the other three benchmarks the test is inconclusive at every band we report, so a benefit of practical size is unobserved rather than excluded. The mixture itself is a separate question. On WebArena vLLM's 66/34 cheap/mid split produces 0.150 success at \$0.026 candidate spend, an intermediate operating point that no single fixed tier reproduces (Always-Cheapest 0.110 at \$0.002, Always-Mid 0.220 at \$0.087), and it is nondominated there with probability 0.997. For the three near-constant routers, observed success is determined by their realized tier assignments by construction. For vLLM, equivalence to a share-matched assignment is established only for WebArena success at the $\pm0.05$ margin; targeting remains unresolved elsewhere. This is not an artifact of scoring RouteLLM and LiteLLM on a three-tier axis they cannot fully use: restricted to their native binary cheap-versus-strong action space (neither can select mid by construction), both still choose cheap on every one of their 580 routes, so the wider action space is not what makes them fixed-tier selectors. One scope note follows. LiteLLM Router is a production gateway whose stated purpose is reliability, budget enforcement, and fallback under provider failure and load, none of which this capability-focused evaluation induces. Its cheap-tier result measures the axis we test, not the axis it is built for, and we read it as a cost floor on the capability question rather than as evidence the gateway is poorly engineered.

\textbf{Fixed-tier baselines are competitive comparators.} Always-Strongest beats every router on RouterBench and ties the best on WebArena; Always-Cheapest leads BFCL; and once Always-Mid is on the table the tau2-bench case where a router (Aurelio) appeared to edge every fixed tier shrinks to a 0.003 margin over Always-Mid, with the 0.016 margin over Always-Strongest unresolved; the transport-failure cascade is identical to Always-Cheapest.

Because that Always-Mid comparison carries the paper's routing-behavior finding, we test it directly rather than resting on point estimates and an interval against a different baseline. Table \ref{tab:policy-equivalence} reports the paired per-task difference between Aurelio and Always-Mid on every benchmark, with a hierarchical paired 90\% interval, a two-sided task-level sign-flip test under a symmetry assumption and a TOST equivalence test against the protocol-declared $\pm0.05$ margin, both Holm-adjusted across the four benchmarks. On RouterBench, BFCL, and WebArena, the comparison is not a statistical one: Aurelio selects mid-general on every task, making its per-task difference against Always-Mid identically zero. Because these per-task differences on these benchmarks are identically zero, these cases are reported as exact identities without bootstrap intervals. tau2-bench is the only informative case, since it is the only benchmark where Aurelio makes some non-mid selections (3.0\% of routes). Even there, the difference is $+0.003$ (hierarchical paired 90\% interval [0.000, 0.010]), is bounded within the margin at a Holm-adjusted TOST $p<0.001$, and traces to a single task out of 100 whose graded outcome differs at all. So the minority of content-dependent decisions Aurelio does make on tau2-bench buys a difference that is bounded within five success-rate points and attributable to a single task. If a router behaves as a fixed-tier selector, comparing it to an unconditional tier baseline is an identity comparison. This is why the paper reports the baselines and the selected-candidate distribution before the router comparison: without them, Aurelio's three benchmark wins read as evidence of good routing rather than a fortunate fallback default.

\textbf{Rank inversion appears in the point estimates; only its RouterBench leg is statistically resolved.} Proposition \ref{prop:rank} says per-benchmark reporting is mandatory once two routers' success rates invert, and the point estimates show exactly that: Aurelio leads three benchmarks and sits below the cheap-only routers on BFCL. Only one leg of that reversal is resolved, however: Aurelio's RouterBench lead is statistically distinguishable under the primary adjusted test, whereas its BFCL deficit is not (a $-0.022$ gap, $p=1.000$, Aurelio below LiteLLM in only 58\% of bootstrap draws), so the population-level inversion is directional rather than established. The direction of the reversal is firmer than the four-way ranking is statistically settled at these sample sizes (RouterBench and BFCL carry 60 and 30 tasks, tau2-bench and WebArena 100 each), so we read the ordering as directional where the marginal intervals overlap and report the paired risk differences with hierarchical task-clustered bootstrap intervals and Holm-adjusted task-level sign-flip $p$-values throughout (Section \ref{sec:estimand}); the share-matched tier-label permutation reference analysis is used only for vLLM targeting. Under that unit-consistent test, the established router-pair effects are Aurelio (and thus the mid tier) over the cheap-only routers on RouterBench, tau2-bench and WebArena, Aurelio over vLLM on RouterBench and on tau2-bench, and vLLM over the cheap-only routers on tau2-bench; the remaining WebArena pairs and the whole BFCL ordering are directional but not significant after multiplicity control. Every one of those established effects is a tier gap rather than a targeting gap. On RouterBench, tau2-bench and WebArena alike Aurelio's selection is effectively constant at mid-general, so an effect involving it compares the mid tier against whatever the other policy holds, whether that is the cheap tier or vLLM's mixture. What the tests resolve is that the mid tier beats the cheap tier where the candidate pool separates them, and that the resolution runs out where it does not.

\textbf{Auxiliary pool-swap lineage; not part of the canonical four-benchmark comparison.} The nominal cheap slot was replaced, so model, provider, price, and associated metadata changed together. Table \ref{tab:pool-ablation} reports this separate two-benchmark lineage (RouterBench and BFCL v4, an earlier run under the same broad protocol). Static routers do not move their selections: Aurelio's success is unchanged to four decimals (0.758), and LiteLLM and RouteLLM remain cheap-locked. Their cost per task rises roughly $16\times$ (LiteLLM \$0.000038 $\to$ \$0.000618; RouteLLM \$0.000037 $\to$ \$0.000615); observed success does not improve and is slightly lower, but this descriptive difference is not tested. vLLM Semantic Router, whose mix spans tiers, sees its per-task cost move only $1.32\times$. The supported claim is limited to price propagation for policies locked to the replaced slot; this lineage does not identify model capability or causal routing benefit.

This swap does not establish that the cheap-to-mid capability gap has narrowed, and separating the two benchmarks shows why: they are different experiments rather than two instances of one. On RouterBench (Table \ref{tab:pool-ablation-tiers-routerbench}) the replay is keyed by tier rather than by model identity, so not one of its 480 rows can change success under a cheap-slot swap: its cost moves and its success cannot, by construction. That half is a price-slot relabeling and carries no information about the replacement model's capability at all. Only BFCL (Table \ref{tab:pool-ablation-tiers-bfcl}) executes the replacement model, so only BFCL can speak to capability, and there the swapped cheap tier falls from 0.846 to 0.827, a change of 7 rows out of 162, against the unchanged mid tier's own re-execution jitter of 2 rows out of 64. A 7-row shift on one 30-task benchmark does not establish that a capability gap narrowed, not with a 2-row noise floor on a tier whose model did not change and with the provider and model metadata changing together. The direction is wrong for it anyway: measured against the fixed mid tier, the observed cheap-to-mid gap on BFCL widened rather than closed. ``Near-mid'' describes the replacement model's price and market position, not anything measured here. The replay results in Table \ref{tab:pool-ablation-tiers-routerbench} represent a price-slot relabeling and should not be interpreted as validation of the live execution results in Table \ref{tab:pool-ablation-tiers-bfcl}.

One further limitation of those two tables is structural and we can bound its size. Each tier's row is pooled over whichever routers selected that tier, so different tiers are summarized on different task subsets, and a per-tier comparison built that way is selection-dependent rather than a capability measurement. The correct construction uses every task-candidate cell paired by task and replicate, and that construction is unavailable for these lineages: neither ablation bundle records a candidate matrix, only routed rows, so the unrouted cells were never written and cannot be recovered. Table \ref{tab:tier-capability} quantifies this distortion on the canonical bundle, which contains the complete candidate matrix. The result is sharply uneven. Where every task is routed to a tier the conditioning is vacuous and the selection effect is exactly zero; this covers the cheap and mid rows where every task is routed to that tier, while tau2-bench mid is nearly, but not exactly, constant. Where routing is sparse it is severe: BFCL's strong-frontier tier scores 0.800 over all 30 tasks but 0.571 over the 7 tasks any router actually sent there, an artifact of $-0.229$, and WebArena's strong-frontier tier is never selected at all, so no conditioned value for it exists. The lesson transfers directly to the ablation tables: their densely routed cheap and mid rows are usable as descriptive summaries of selected subsets, not as causal tier-capability estimates, and their sparse strong-frontier rows (14 and 20 route rows) must not be read as tier capability in either pool. We report them for completeness and draw no tier comparison from them.

Accordingly, these tables are descriptive rather than a controlled ablation, as swapping only the cheap slot's price does not isolate content-blindness from other variables that change with a model swap. A capability-, price-, provider-, and metadata-controlled ablation that would license a causal claim is left to future work; here we report only that the cheap-locked policies pass the exogenous price change straight to the bill. (Because this ablation uses a different, two-benchmark pool, its absolute numbers are not comparable to the four-benchmark canonical run and are reported only for the within-ablation contrast.)

\begin{table*}[t]
\caption{Direct paired comparison of Aurelio against Always-Mid, the fixed tier it reproduces, on the candidate-cost basis. \emph{Diff.\ tasks} is the number of tasks whose graded outcome differs at all. Where Aurelio selects mid-general on every task the difference is an exact identity rather than an estimate, so no interval is reported. Sup.\ $p$ is the two-sided task-level sign-flip test under a symmetry assumption and TOST $p$ the equivalence test against the protocol-declared $\pm0.05$ margin, both Holm-adjusted across the four benchmarks.}
\label{tab:policy-equivalence}
\centering
\small
\begin{tabular}{@{}lrrlllll@{}}
\toprule
Bench & Tasks & Diff.\ tasks & $\Delta$ success & 90\% interval & Sup.\ $p$ & TOST $p$ & Verdict \\
\midrule
BFCL v4 & 30 & 0 & 0.000 & exact & --- & --- & identical by construction \\
RouterBench & 60 & 0 & 0.000 & exact & --- & --- & identical by construction \\
WebArena & 100 & 0 & 0.000 & exact & --- & --- & identical by construction \\
tau2-bench & 100 & 1 & 0.003 & [0.000, 0.010] & $>$0.999 & $<$0.001 & equivalent \\
\bottomrule
\end{tabular}

\end{table*}

\begin{table*}[t]
\caption{Candidate-pool ablation (auxiliary two-benchmark lineage): swapping only the cheap tier for a more expensive near-mid model. Static routers hold their selections, so their cost tracks the exogenous slot price.}
\label{tab:pool-ablation}
\centering
\small
\begin{tabular}{@{}lrrrrr@{}}
\toprule
Router & Wide succ. & Narrow succ. & Wide USD/task & Narrow USD/task & Narrow/Wide cost \\
\midrule
Aurelio Semantic Router & 0.758 & 0.758 & 0.001934 & 0.001940 & 1.00 \\
LiteLLM Router & 0.567 & 0.550 & 0.000038 & 0.000618 & 16.33 \\
RouteLLM & 0.558 & 0.550 & 0.000037 & 0.000615 & 16.51 \\
vLLM Semantic Router & 0.600 & 0.600 & 0.001162 & 0.001538 & 1.32 \\
\bottomrule
\end{tabular}

\end{table*}

\begin{table}[t]
\caption{RouterBench half of the pool ablation: a \emph{price-slot relabeling}, not a model ablation. $\ast$ marks the swapped tier. The replay is keyed by tier rather than model identity, so none of these 480 rows can change success by construction; only cost moves. No capability claim is available from this half.}
\label{tab:pool-ablation-tiers-routerbench}
\centering
\small
\resizebox{\columnwidth}{!}{\begin{tabular}{@{}lrrrrr@{}}
\toprule
Tier & n & Wide & Narrow & $\Delta$ & Rows changed \\
\midrule
cheap-small $\ast$ & 338 & 0.308 & 0.308 & +0.000 & 0 \\
mid-general & 122 & 0.672 & 0.672 & +0.000 & 0 \\
strong-frontier & 20 & 0.800 & 0.800 & +0.000 & 0 \\
\bottomrule
\end{tabular}
}
\end{table}

\begin{table}[t]
\caption{BFCL half of the pool ablation: the only half that executes the replacement model, on one 30-task benchmark with the provider and model metadata changing together. $\ast$ marks the swapped tier; mid and strong are unchanged models, so their rows set the re-execution noise floor. Rows are pooled over whichever routers selected the tier, so the sparse strong-frontier row is selection-dependent and is not a tier-capability measurement (Table \ref{tab:tier-capability}).}
\label{tab:pool-ablation-tiers-bfcl}
\centering
\small
\resizebox{\columnwidth}{!}{\begin{tabular}{@{}lrrrrr@{}}
\toprule
Tier & n & Wide & Narrow & $\Delta$ & Rows changed \\
\midrule
cheap-small $\ast$ & 162 & 0.846 & 0.827 & -0.018 & 7 \\
mid-general & 64 & 0.812 & 0.812 & +0.000 & 2 \\
strong-frontier & 14 & 0.571 & 0.571 & +0.000 & 0 \\
\bottomrule
\end{tabular}
}
\end{table}

\begin{table*}[t]
\caption{How much a route-conditioned per-tier summary distorts a per-tier capability read, measured on the canonical bundle because it is the only lineage carrying a complete candidate matrix. \emph{Success (all)} uses every task-candidate cell for the tier; \emph{Success (routed)} restricts to tasks some router selected that tier for, reproducing the construction of Tables \ref{tab:pool-ablation-tiers-routerbench} and \ref{tab:pool-ablation-tiers-bfcl}. The effect is exactly zero wherever every task is routed to the tier and large wherever routing is sparse, which is why the sparse rows of those tables carry no capability claim.}
\label{tab:tier-capability}
\centering
\small
\resizebox{\textwidth}{!}{\begin{tabular}{@{}llrlrll@{}}
\toprule
Bench & Tier & All tasks & Success (all) & Routed tasks & Success (routed) & Selection effect \\
\midrule
BFCL v4 & cheap-small & 30 & 0.833 & 30 & 0.833 & 0.000 \\
BFCL v4 & mid-general & 30 & 0.811 & 30 & 0.811 & 0.000 \\
BFCL v4 & strong-frontier & 30 & 0.800 & 7 & 0.571 & -0.229 \\
RouterBench & cheap-small & 60 & 0.300 & 60 & 0.300 & 0.000 \\
RouterBench & mid-general & 60 & 0.683 & 60 & 0.683 & 0.000 \\
RouterBench & strong-frontier & 60 & 0.800 & 10 & 0.800 & 0.000 \\
WebArena & cheap-small & 100 & 0.110 & 100 & 0.110 & 0.000 \\
WebArena & mid-general & 100 & 0.220 & 100 & 0.220 & 0.000 \\
WebArena & strong-frontier & 100 & 0.220 & 0 & --- & n/a \\
tau2-bench & cheap-small & 100 & 0.543 & 100 & 0.543 & 0.000 \\
tau2-bench & mid-general & 100 & 0.810 & 97 & 0.808 & -0.002 \\
tau2-bench & strong-frontier & 100 & 0.797 & 20 & 0.850 & 0.053 \\
\bottomrule
\end{tabular}
}
\end{table*}

\textbf{Threshold reconfiguration can change the assignment substantially; this experiment does not establish a quality benefit from it.} The two routers with an exposed threshold parameter were re-measured across it, to test how much of their near-constant behavior is attributable to a shipped default. Tables \ref{tab:threshold-sweep-rb} and \ref{tab:threshold-sweep-bfcl} report the sweep: RouteLLM's strong-versus-weak escalation threshold (default 0.50) and Aurelio's semantic score threshold (default 0.30), on the canonical pool and frozen task set with only the threshold varied.

One point of construction governs how the sweep is scored. The sweep lineage executed live model calls separately at each threshold, so its own graded outcomes carry per-threshold generation noise, under which two thresholds that select an identical tier for every task can still report different success and different cost. We therefore take only the route decisions from that lineage and score every threshold against the single locked candidate matrix, joining each decision to the frozen outcome and cost for the tier it selected. Success and cost are then a deterministic function of the tier vector, so identical decisions give identical numbers by construction. The analysis asserts this rather than assuming it, together with the requirement that each router's shipped-default row reproduce its canonical per-benchmark success exactly. Validation holds: the sweep's default-threshold decisions match the canonical routes on 360 RouterBench rows and 180 BFCL rows (540 joined rows total); routing is reproducible, while outcome generation is not. The family is 14 comparisons against a default: two benchmarks, with four off-default thresholds for RouteLLM and three for Aurelio on each. Both the superiority and the equivalence tests are Holm-adjusted across that family, and the sweep is reported as exploratory. Its superiority procedure is the same task-level sign-flip test under a symmetry assumption that the main analysis uses, but its interval differs in one respect we state rather than leave for a reader to find: because a re-scored decision joins to the frozen matrix rather than carrying its own replicates, the sweep resamples tasks from their per-task means without the inner replicate stage of Section \ref{sec:estimand}. That omits the within-task output variation, which is a further reason to read these intervals as exploratory rather than as comparable to the main tables.

Two distinct behaviors appear, and only one of them produces a resolved quality change. RouteLLM \emph{relocates} wholesale. Its RouterBench success is bimodal, 0.300 at every threshold at or above 0.30 and 0.800 at every threshold at or below 0.20, that is exactly Always-Cheapest or exactly Always-Strongest with no intermediate operating point: at the 0.30-to-0.20 step all 60 tasks flip together from cheap-small to strong-frontier, mid-general is selected at no threshold, and no threshold produces a mixed split. The paired difference at that step is large and survives the correction ($+0.500$, 90\% interval [0.367, 0.633], Holm-adjusted $p=0.001$), so reconfiguration here genuinely does buy quality. It buys it by becoming a different fixed tier, not by discriminating between tasks. The same relocation on BFCL moves the wrong way and is unresolved ($-0.033$ [$-0.133$, 0.067]), which is what a tier switch looks like on the benchmark where cheap is already the best tier.

Aurelio on BFCL shows the other behavior, genuine \emph{reassignment}: lowering the cutoff moves the tier for 5, 18, and 19 of 30 tasks, up to two-thirds of the benchmark, with real three-way splits (0.40/0.37/0.23 at the lowest cutoff). This represents task-level variation in the assignment rather than a relabeled fixed tier, demonstrating that the shipped default is what makes Aurelio's canonical action constant: lowering the threshold produces a mixed policy. What the reassignment does not produce is a resolved change in graded outcome. The paired differences are $-0.033$ [$-0.100$, 0.000] and 0.000 [$-0.067$, 0.067] twice, every one inconclusive after adjustment at 30 tasks, so a benefit of practical size is neither demonstrated nor excluded. On RouterBench, by contrast, Aurelio barely moves at all (3 to 4 of 60 tasks reassigned, success 0.683 to 0.700, $+0.017$ [$-0.033$, 0.067]): there its binding constraint is reference-utterance coverage rather than the score cutoff, since three short utterances per tier cannot match these prompts at any threshold. So the coverage-not-cutoff reading is specific to RouterBench and does not extend to BFCL, where the cutoff plainly does bite.

The sweep therefore supports a narrower conclusion than either ``reconfiguration cannot create routing'' or ``the default is irrelevant.'' Reconfiguration can create materially variable task assignments, and on Aurelio's BFCL leg it does; what this experiment does not establish is a reliable quality benefit from those assignments. The one large resolved gain comes from relocating to a different fixed tier, and the one case of substantial per-task reassignment leaves the graded outcome unresolved. Whether a finer threshold could help on BFCL is not settled by 30 tasks. Note also what the unchanged-decision rows are: thresholds that change no routing decision at all, whose paired difference is therefore an exact identity rather than merely small.

\begin{table*}[t]
\caption{Operating-point sweep on RouterBench. Shipped defaults in bold. C/M/S is the route share per tier; \emph{Chg.} the number of tasks differing from default. Success is joined to the locked canonical matrix; route-level candidate costs are retained in the associated artifact. Sup.\ $p$ and TOST $p$ are Holm-adjusted across the sweep's 14 non-default comparisons ($\pm0.05$ margin).}
\label{tab:threshold-sweep-rb}
\centering
\small
\resizebox{\textwidth}{!}{\begin{tabular}{@{}llrrlllll@{}}
\toprule
Router & Thr. & C/M/S & Chg. & Success [95\%] & Paired $\Delta$ vs default [90\%] & Sup.\ $p$ & TOST $p$ & Verdict \\
\midrule
Aurelio & \textbf{0.30} & 0.00/1.00/0.00 & 0 & 0.683 [0.567, 0.800] & --- & --- & --- & (default) \\
Aurelio & 0.20 & 0.03/0.95/0.02 & 3 & 0.700 [0.583, 0.817] & 0.017 [-0.033, 0.067] & $>$0.999 & $>$0.999 & inconclusive \\
Aurelio & 0.10 & 0.03/0.93/0.03 & 4 & 0.700 [0.583, 0.817] & 0.017 [-0.033, 0.067] & $>$0.999 & $>$0.999 & inconclusive \\
Aurelio & 0.05 & 0.03/0.93/0.03 & 4 & 0.700 [0.583, 0.817] & 0.017 [-0.033, 0.067] & $>$0.999 & $>$0.999 & inconclusive \\
RouteLLM & \textbf{0.50} & 1.00/0.00/0.00 & 0 & 0.300 [0.183, 0.417] & --- & --- & --- & (default) \\
RouteLLM & 0.40 & 1.00/0.00/0.00 & 0 & 0.300 [0.183, 0.417] & 0.000 [0.000, 0.000] & $>$0.999 & $<$0.001 & equivalent \\
RouteLLM & 0.30 & 1.00/0.00/0.00 & 0 & 0.300 [0.183, 0.417] & 0.000 [0.000, 0.000] & $>$0.999 & $<$0.001 & equivalent \\
RouteLLM & 0.20 & 0.00/0.00/1.00 & 60 & 0.800 [0.700, 0.900] & 0.500 [0.367, 0.633] & 0.001 & $>$0.999 & higher success \\
RouteLLM & 0.10 & 0.00/0.00/1.00 & 60 & 0.800 [0.700, 0.900] & 0.500 [0.367, 0.633] & 0.001 & $>$0.999 & higher success \\
\bottomrule
\end{tabular}
}

\end{table*}

\begin{table*}[t]
\caption{Operating-point sweep on BFCL v4, columns as in Table \ref{tab:threshold-sweep-rb}. This is the leg where a threshold change produces genuine task-level reassignment rather than a relabeled fixed tier: Aurelio moves the tier for up to 19 of 30 tasks with a real three-way split, and every resulting difference in graded outcome is inconclusive after adjustment.}
\label{tab:threshold-sweep-bfcl}
\centering
\small
\resizebox{\textwidth}{!}{\begin{tabular}{@{}llrrlllll@{}}
\toprule
Router & Thr. & C/M/S & Chg. & Success [95\%] & Paired $\Delta$ vs default [90\%] & Sup.\ $p$ & TOST $p$ & Verdict \\
\midrule
Aurelio & \textbf{0.30} & 0.00/1.00/0.00 & 0 & 0.811 [0.667, 0.933] & --- & --- & --- & (default) \\
Aurelio & 0.20 & 0.03/0.83/0.13 & 5 & 0.778 [0.622, 0.911] & -0.033 [-0.100, 0.000] & $>$0.999 & $>$0.999 & inconclusive \\
Aurelio & 0.10 & 0.37/0.40/0.23 & 18 & 0.811 [0.667, 0.933] & 0.000 [-0.067, 0.067] & $>$0.999 & $>$0.999 & inconclusive \\
Aurelio & 0.05 & 0.40/0.37/0.23 & 19 & 0.811 [0.667, 0.933] & 0.000 [-0.067, 0.067] & $>$0.999 & $>$0.999 & inconclusive \\
RouteLLM & \textbf{0.50} & 1.00/0.00/0.00 & 0 & 0.833 [0.700, 0.967] & --- & --- & --- & (default) \\
RouteLLM & 0.40 & 1.00/0.00/0.00 & 0 & 0.833 [0.700, 0.967] & 0.000 [0.000, 0.000] & $>$0.999 & $<$0.001 & equivalent \\
RouteLLM & 0.30 & 1.00/0.00/0.00 & 0 & 0.833 [0.700, 0.967] & 0.000 [0.000, 0.000] & $>$0.999 & $<$0.001 & equivalent \\
RouteLLM & 0.20 & 0.00/0.00/1.00 & 30 & 0.800 [0.633, 0.933] & -0.033 [-0.133, 0.067] & $>$0.999 & $>$0.999 & inconclusive \\
RouteLLM & 0.10 & 0.00/0.00/1.00 & 30 & 0.800 [0.633, 0.933] & -0.033 [-0.133, 0.067] & $>$0.999 & $>$0.999 & inconclusive \\
\bottomrule
\end{tabular}
}

\end{table*}

\phantomsection
\subsection*{Threats to Validity}
The main external-validity threat is scope. All statistical summaries are conditional on one author-defined three-model candidate pool and frozen task samples: BFCL contributes 30 sampled tasks, while tau2-bench and WebArena use deterministic convenience prefixes. Tier labels jointly encode model family, provider, price, API behavior, prompt compatibility, and tool-call conventions, so the study cannot identify which component causes a tier effect or generalize the observed ordering to other pools. The auxiliary pool analysis is not controlled and does not remove this threat. Likewise, the central fixed-tier behavior characterizes evaluated adapters and defaults rather than intrinsic properties of each package: Aurelio's reference utterances, easy/medium/hard mapping, and mid-tier fallback are author supplied, while RouteLLM's default threshold was not calibrated to this candidate pair. A held-out calibration or factorial configuration study would be needed to support vendor-recommended or configuration-optimized comparisons. Replay also excludes router-specific prompt transformations, post-generation fallback, provider-load effects, and route-conditioned trajectories. Accordingly, empirical claims concern these router configurations, adapters, candidate models, and evaluation controls; they are not claims about routing paradigms in general. Broader task samples, tuned-router comparisons, and additional candidate pools are needed to assess external validity.

The foundational threat is adapter fidelity: each adapter must reflect its router's real mechanism, not the integration code's distortion of it. We mitigated this by checking per-task \texttt{selected\_candidate}, \texttt{confidence}, and \texttt{fallback\_path} traces against each adapter's source, not just aggregate tables, and it mattered. This check caught a failure in the evaluated vLLM deployment path: when its local stack was unreachable, the service accepted the routing configuration but returned a single-tier default rather than executing the intended routing behavior. Every vLLM number here comes from the corrected deployment, verified directly against the service's own decision logs, with \texttt{router\_service\_error} occurring in 0 of 2{,}320 locked route rows. Because vLLM and the other routers are live, externally maintained services that can change after this paper's run, any future use of this artifact should repeat the same audit rather than assume today's fix still holds. The second threat is sample-size heterogeneity: the task-clustered inference rests on 30, 60, 100, and 100 frozen tasks (180, 360, 600, and 600 joined outcomes) for BFCL, RouterBench, tau2-bench, and WebArena; the task-level paired-bootstrap CIs report its consequence directly, and the four-way ranking is directional where they overlap. A third threat is that the cross-benchmark consistency claim rests on only three or four suites, so ``no router combines best mean rank with lowest variance'' is directional evidence from these benchmarks, not a general property. A fourth threat is construct validity of the routing granularity: every router decides once per task or session (Table \ref{tab:granularity}), so this study evaluates task/session-level model selection, not per-step in-trajectory routing. It therefore does not test whether a router adapts to a failed action, a tool response, later-trajectory difficulty, or its own intermediate confidence, all capabilities a per-step condition would probe. Establishing those would require a per-step routing mode that none of the four evaluated adapters currently implements, so the results here bound only the coarser decision the routers actually expose.

\section{Conclusion}
\label{sec:conclusion}

We presented a common interface, deployment-focused metric suite, and hybrid protocol for evaluating task- and session-level model routing across four router implementations and four benchmarks. Main result is not universally best router. Peak success, cross-benchmark consistency, and task-specific targeting are distinct questions, and fixed-tier controls expose those differences. Under evaluated configurations, candidate pool, adapters, and frozen samples, Always-Mid matches Aurelio exactly on three benchmarks and within 0.003 on the fourth. vLLM is only router with substantial prompt-dependent tier variation, yet no benchmark shows resolved task-specific superiority over its share-matched content-blind allocation. Observed gains therefore track selected-tier composition more closely than demonstrated targeting quality.

This is a measurement result, not a verdict on routing paradigms. External validity is limited by 30 BFCL tasks, deterministic tau2-bench and WebArena prefixes, one three-model candidate pool, package defaults and author-supplied adapter settings, an uncalibrated RouteLLM threshold, and replay conditions that omit prompt transformations, post-generation fallback, provider-load effects, and route-conditioned trajectories. Stronger claims require broader task samples, additional candidate pools, tuned-router comparisons, and per-step evaluations. Within its declared scope, the practical recommendation is clear: report fixed-tier baselines, selected-tier distributions, benchmark-specific uncertainty, and separate router from candidate costs before attributing aggregate gains to routing.

\section*{Data and Code Availability}
The evaluation harness, frozen task identifiers, router configurations, aggregate candidate outcomes, analysis protocol, checksums, and reproduction scripts are available at \url{https://github.com/knkumar/router_benchmark}. Raw provider outputs are distributed subject to applicable licensing and usage restrictions.

\appendices
\section{Reproducibility}
The appendix documents the locked artifact and validation lineage.
Every number in this paper is reproducible from the artifact, from a named frozen lineage in each case: the canonical four-benchmark bundle regenerates the main tables, while each auxiliary result (the two-benchmark pool ablation and the operating-point threshold sweep) is reproducible from its own separately manifested lineage rather than the canonical bundle. A single \texttt{reproduce\_canonical} entry point regenerates the analysis tables from the locked bundle under its recorded protocol, so the canonical build is not sensitive to which bundle a caller happens to point at, and each auxiliary lineage ships its own manifest and checksums. The pipeline is tied to fixed model IDs (cheap-small $=$ gpt-5.4-nano, mid-general $=$ claude-sonnet-4-6, strong-frontier $=$ claude-opus-4-8), real dated per-token pricing (snapshot 2026-07-02), and frozen task IDs (RouterBench 60, BFCL v4 30, tau2-bench 100, WebArena 100), whose eligible populations, selection methods, and seeds are in Table \ref{tab:task-selection}. Canonical validation checks schema, frozen coverage, trace lineage, named costs, and \texttt{cache\_flag=false} for every candidate row.

The locked bundle is \texttt{paper1\_\allowbreak{}canonical\_\allowbreak{}webarena\_\allowbreak{}repair\_\allowbreak{}v2} under protocol \texttt{paper1-\allowbreak{}rebuild-\allowbreak{}webarena-\allowbreak{}repair-\allowbreak{}v2}, created 2026-07-23T22:57:37Z, and it ships a SHA-256 checksum for each of its eight files. The two that carry every reported outcome are \texttt{candidate\_\allowbreak{}outcomes.csv} (\texttt{f5fc349f}\ldots) and \texttt{routes.csv} (\texttt{54898cad}\ldots); the remaining six files are checksummed alongside them in \texttt{checksums.\allowbreak{}sha256}, which is the file to verify a copy against. Table \ref{tab:package-versions} pins the software the run executed under. The evaluation environment is Python 3.11.0; the four router packages and the two graders whose versions can change a reported number are listed there, and the WebArena harness runs in a separate pinned environment because its dependency stack is incompatible with the main one. All model identifiers, the 2026-07-02 pricing snapshot, per-router service-fee bases, and the embedding timeout and retry settings applied to Aurelio and RouteLLM are recorded in the bundle's \texttt{router\_configs.json}; the frozen task IDs, seeds, and per-benchmark route and outcome row counts are in \texttt{manifest.json}.

\begin{table}[t]
\caption{Pinned software versions for the canonical run. The WebArena harness runs in a separate vendored environment (\texttt{openai==0.27.0}, \texttt{transformers==4.33.2}) because its stack is incompatible with the main one; its browser build is the binding constraint repaired before this run.}
\label{tab:package-versions}
\centering
\small
\resizebox{\columnwidth}{!}{\begin{tabular}{@{}lll@{}}
\toprule
Component & Package & Version \\
\midrule
RouteLLM & \texttt{routellm} & 0.2.0 \\
Aurelio Semantic Router & \texttt{semantic-router} & 0.1.15 \\
LiteLLM Router & \texttt{litellm} & 1.90.3 \\
vLLM Semantic Router & \texttt{vllm-sr} & 0.3.0 \\
\midrule
BFCL v4 grader & \texttt{bfcl-eval} & 2026.3.23 \\
WebArena browser & \texttt{playwright} & 1.32.1 (Chromium 1055) \\
\midrule
Model clients & \texttt{anthropic} / \texttt{openai} & 0.116.0 / 2.44.0 \\
Analysis stack & \texttt{numpy} / \texttt{pandas} & 1.26.4 / 3.0.3 \\
\bottomrule
\end{tabular}
}
\end{table}

Observed candidate-generation spend is \$299.04753042, router-service spend \$40.60000000, and external metered spend \$23.62602300, for \$363.27355342 total; infrastructure cost was not recorded. The external metered figure is reconciled to policies rather than left as an aggregate: every cent of it is tau2-bench user-simulator model spend, recorded per candidate cell, and it is charged per selected cell under the third cost basis of Section \ref{sec:expected-utility}. Because it is benchmark-execution cost that varies with selected agent tier, excluding it would change reported benchmark-cost accounting and favor policies that produce longer simulated sessions. The per-cell ledger lives in the execution stage that produced the bundle rather than in the bundle schema, so the allocation step validates that its cell keys match the locked candidate matrix exactly (2{,}610 for 2{,}610) before joining. Reproducibility artifacts include the frozen protocol, bundle manifest and checksums, router configurations, saved paired-bootstrap draws, baseline reconciliation, BFCL route equivalence, the selected-candidate distribution, the all-policy Pareto frontier under both cost bases, the utility upper-envelope breakpoints, the share-matched permutation and paired equivalence output, the per-tier candidate-pool diagnostics, the paired threshold-sweep re-analysis and its lineage manifest, the oracle/cascade, regret, and success-per-dollar analyses, and the observed spend ledger, generated from the locked bundle. The subsequent appendix section reports generated full-rebuild evidence from the locked bundle, including the execution matrix, baseline tables, paired effects, rank and Pareto uncertainty, route-equivalence output, and artifact manifest.

\subsection{Router Configurations}
The canonical run uses each router's package defaults, wired through our specified adapters, against the common three-tier pool. RouteLLM runs the \texttt{sw\_ranking} win-rate predictor over the task prompt and routes to strong-frontier iff the predicted strong win-rate meets its escalation threshold, left at the package default 0.50; the pool exposes cheap-small and strong-frontier as its weak/strong endpoints. Aurelio Semantic Router uses \texttt{semantic-router} with the \texttt{text-embedding-3-small} encoder and three reference utterances per tier under an easy/medium/hard framing, per-route \texttt{score\_threshold} 0.30 and ``max'' score aggregation; when no tier's utterances clear the threshold it takes its fallback default (mid-general), which is why its fallback rate is high on benchmarks whose prompts its utterance set does not cover. LiteLLM Router uses \texttt{cost-based-routing}, which selects by configured \$/token only and reads no prompt content. vLLM Semantic Router uses its Mixture-of-Models probe over the prompt. Each \texttt{route()} sees the task prompt, the candidate pool, and each candidate's configured cost; no router is given the benchmark identity, the grader, or any held-out label. The two thresholds are the only configurable escalation parameters; both are left at their package defaults rather than tuned on a held-out split, so the operating-point sweep in Section \ref{sec:discussion} reports the full range of their effect directly instead of a single calibrated value. ``Shipped configuration'' here means each router's package defaults wired through our specified adapter, not a vendor-tuned deployment. The three reference utterances per tier, the easy/medium/hard framing, and the mid-general fallback are author-supplied adapter choices, so Aurelio's fallback rate reflects our adapter as much as the package. RouteLLM's default 0.50 threshold is applied to a cheap/strong candidate pair it was not calibrated on. The results therefore characterize these routers under our specified adapters and package defaults, not the best configuration each could reach with per-benchmark tuning. All four \texttt{router\_configs} are shipped in the artifact.

\section{Canonical Full-Rebuild Evidence}
\label{app:canonical-full-rebuild}

This appendix is generated from the locked full-rebuild bundle and its canonical analysis outputs. It is absent from builds that have not completed canonical validation.

\begin{table}[!t]
\caption{Canonical full-rebuild matrix}
\label{tab:canonical-rebuild-summary}
\centering
\resizebox{\columnwidth}{!}{\begin{tabular}{@{}lrrrr@{}}
\toprule
Benchmark & Tasks & Replicates & Candidate rows & Route rows \\
\midrule
BFCL v4 & 30 & 3 & 270 & 240 \\
RouterBench & 60 & 3 & 540 & 480 \\
WebArena & 100 & 3 & 900 & 800 \\
tau2-bench & 100 & 3 & 900 & 800 \\
\bottomrule
\end{tabular}
}
\end{table}

\end{document}